\documentclass{article}

 \usepackage[nonatbib,preprint]{neurips_2020}

\usepackage[utf8]{inputenc} 
\usepackage[T1]{fontenc}    
\usepackage{hyperref}       
\usepackage{url}            
\usepackage{booktabs}       
\usepackage{amsfonts}       
\usepackage{nicefrac}       
\usepackage{microtype}     
\usepackage[misc]{ifsym}
\usepackage{bbding}

\usepackage{amsmath}
\usepackage{amsthm}
\usepackage{dsfont}
\usepackage{color}
\usepackage{graphicx}
\usepackage{appendix}

\newtheorem{theorem}{Theorem}
\newtheorem{definition}{Definition}
\newtheorem{principle}{Principle}

\title{Knowledge Distillation in Wide Neural Networks: Risk Bound, Data Efficiency and Imperfect Teacher}

\author{%
    Guangda Ji \\
    Peking University\\
    \texttt{jiguangda@pku.edu.cn} \\
  \And
    Zhanxing Zhu \Envelope \\
    School of Mathematical Sciences, Peking University \\
    Center for Data Science, Peking University \\
    \texttt{zhanxing.zhu@pku.edu.cn} \\
}

\begin{document}

\maketitle

\begin{abstract}
  Knowledge distillation is a strategy of training a student network with guide of the soft output from a teacher network. It has been a successful method of model compression and knowledge transfer. However, currently knowledge distillation lacks a convincing theoretical understanding. On the other hand, recent finding on neural tangent kernel enables us to approximate a wide neural network with a linear model of the network's random features.   In this paper, we theoretically analyze the knowledge distillation of a wide neural network. First we provide a transfer risk bound for the linearized model of the network. Then we propose a metric of the task's training difficulty, called data inefficiency. Based on this metric, we show that for a perfect teacher, a high ratio of teacher's soft labels can be beneficial. Finally, for the case of imperfect teacher, we find that hard labels can correct teacher's wrong prediction, which explains the practice of mixing hard and soft labels.
\end{abstract}

\section{Introduction}
Deep neural network has been a successful tool in many fields of artificial intelligence. However, we typically require  deep and complex networks and much effort of training to achieve good generalization. Knowledge distillation(KD) is a method introduced in \cite{hinton2015distilling}, which can transfer knowledge from a large trained model (i.e. \textit{teacher}) to another smaller network (i.e. \textit{student}). Through distillation, the student network can achieve better performance than direct training from scratch  \cite{hinton2015distilling}.
The vanilla form of KD in classification problem has a combined loss of soft and hard labels,
\begin{equation*}
  \mathcal{L} = \rho L(\mathbf{y}_\mathrm{s},\mathbf{y}_\mathrm{t}) + (1-\rho) L(\mathbf{y}_\mathrm{s},\mathbf{y}_\mathrm{g})
\end{equation*}
where \(\mathbf{y}_\mathrm{t}\) and \(\mathbf{y}_\mathrm{s}\) are \textit{teacher} and \textit{student}'s soft labels ,  \(\mathbf{y}_\mathrm{g}\)  are \textit{ground truth} labels and \(\rho\) is called \textit{soft ratio}.

Apart from this original form of KD, many variants that share the teacher-student paradigm are proposed. \cite{huang2017like} uses intermediate layers of neural network to perform distillation. \cite{xu2017training} and \cite{yoo2019knowledge} adopt adversarial training to reduce the difference of label distribution between teacher and student. In \cite{lassance2020deep} and \cite{lee2019graph}, the authors consider graph-based distillation. Self distillation, proposed in \cite{furlanello2018born}, distills the student from an earlier generation of student of same architecture. The latest generation can outperform the first generation significantly.

A common conjecture on why KD works is that it provides extra information on the output distribution, and student model can use this ``dark knowledge'' to achieve higher accuracy. However, KD still lacks a convincing theoretical explanation.  \cite{tang2020understanding} argues that KD not only transfers super-class correlations to students, but also gives higher credits to correct samples in gradient flow. \cite{barwey2020extracting} finds that KD enables student to learn more task-related pixels and discard the background in image classification tasks. \cite{mobahi2020self} shows that self-distillation has a regularization effect on student's logits. However, very few works establish a
comprehensive view on the  knowledge distillation, including risk bound, the role of the soft ratio and how efficient the distillation makes use of data.

In this work, we attempt to deal with these issues with the help of neural tangent kernel and wide network linearization, i.e. considering distillation process for linearized neural networks. We focus on the soft ratio \(\rho\) as it serves as a continuous switch between original hard label training and soft label distillation. The main contributions of our work are summarized as follows.
\begin{itemize}
  \item We experimentally observe faster convergence rate of transfer risk with respect to sample size for softer tasks, i.e. with high $\rho$. We try to explain this with a new transfer risk bound for converged linearized student networks, based on distribution in random feature space. We show that the direction of weights converges faster for softer tasks. (Sec.~\ref{sec:transfer_risk_bound})

  \item We introduce a metric on task's difficulty, called \textit{data inefficiency}. Through this metric we show, for a perfect teacher, early stopping and higher soft ratio are beneficial in terms of making efficient use of data. (Sec.~\ref{sec:data_inefficiency})

  \item We discuss the benefits of hard labels in imperfect teacher distillation in the scenario of KD practice. We show that a little portion of hard labels can correct student's outputs pointwisely, and also reduce the angle between student and oracle weight. (Sec.~\ref{sec:imperfect})
\end{itemize}

\paragraph{Related Work}

Our work is built on neural tangent kernel techniques introduced in \cite{jacot2018neural,lee2019wide}. They find that in the limit of infinitely wide network, the Gram matrix of network's random feature tends to a fixed limit called neural tangent kernel (NTK), and also stays almost constant during training. This results in an equivalence of training dynamics between the original network and linear model of network's random features. Therefore  we replace the network with its linear model to avoid the trouble of nonlinearity.
The most related work to ours is \cite{phuong2019towards}. They consider distillation of linear models and gives a loose  transfer risk bound. This bound is based on the probability distribution in feature space and therefore is different form the traditional generalization given by Rademacher complexity. We improve their bound and generalize their formulation to the case of linearization of an actual neural network.

\section{\label{sec:problem_setup}Problem Setup}

\begin{figure}[htbp]
  \centering
  \includegraphics[width=0.32\textwidth]{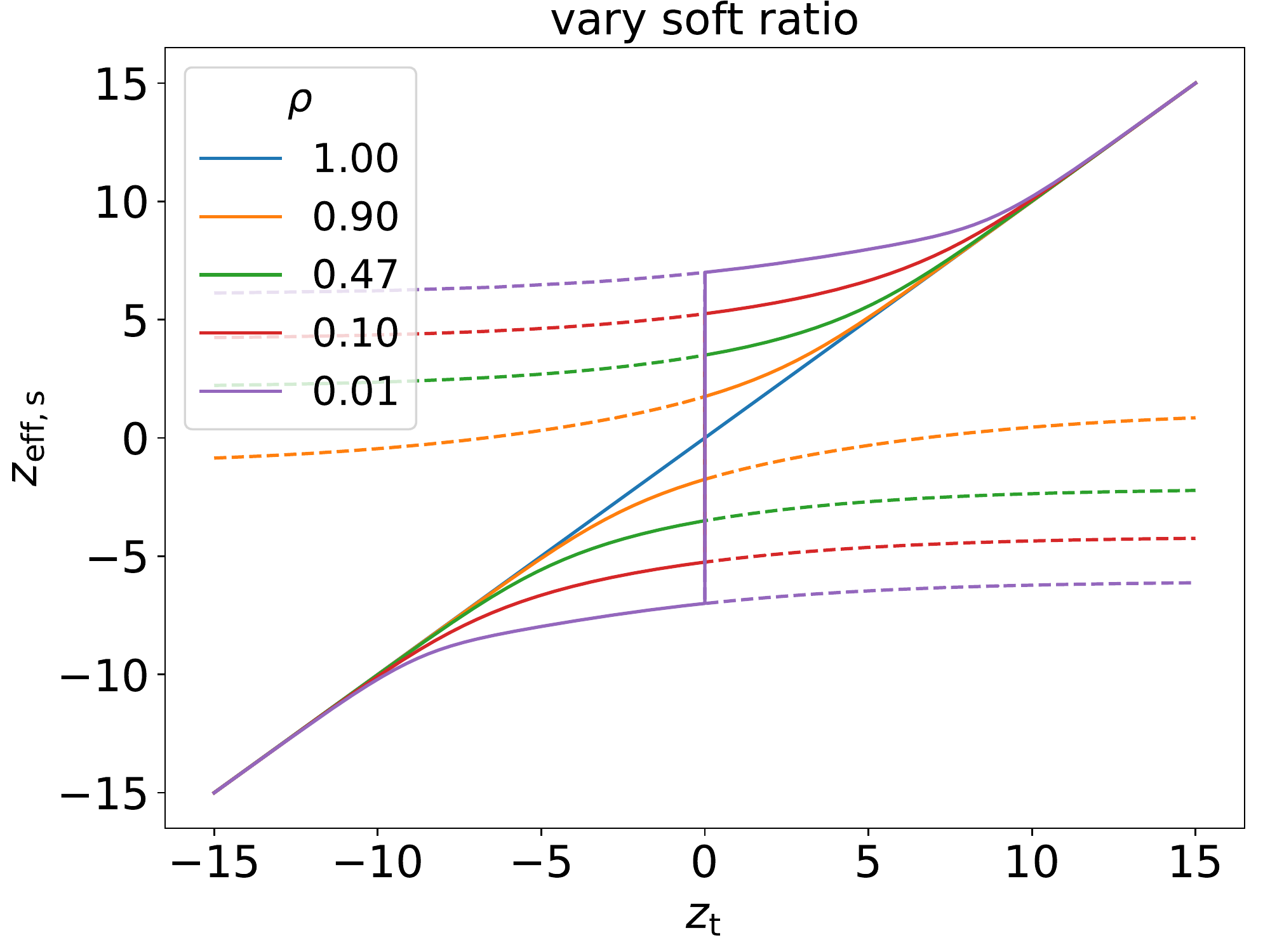}
  \includegraphics[width=0.32\textwidth]{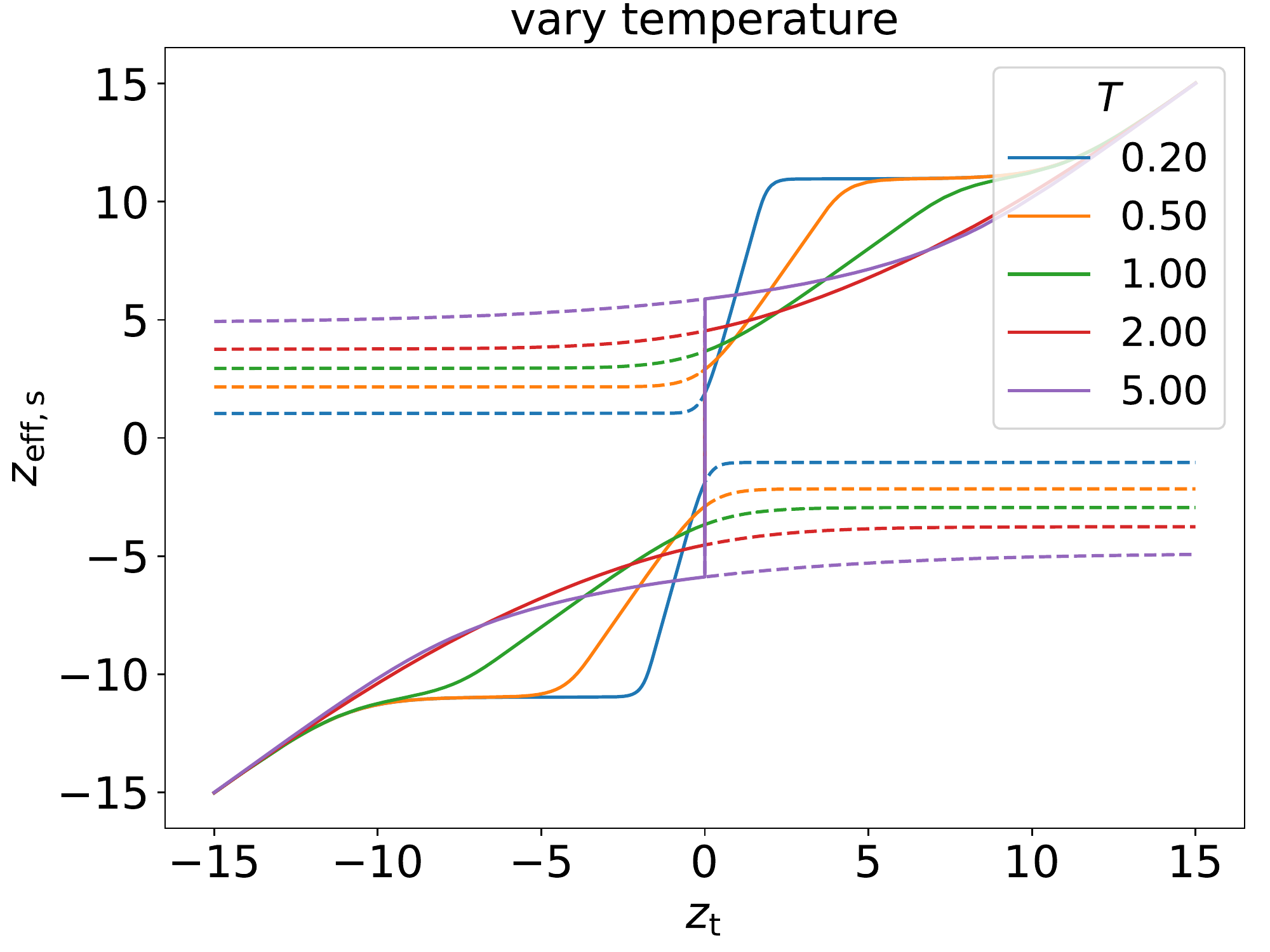}
  \vspace{-0.2cm}
  \caption{\label{fig:1}
    Effective student logits \(z_{\mathrm{s,eff}}\) as a function of  \(z_{\mathrm{t}}\) and  \(y_{\mathrm{g}}\). The \textbf{left} and \textbf{right} figure shows how soft ratio \(\rho\) (with \(T=5.0\)) and temperature \(T\) (with \(\rho=0.05\)) can affect the shape of \(z_{\mathrm{s,eff}}(z_{\mathrm{t}}, y_{\mathrm{g}})\).
    Each point is attained by solving Eq.~\ref{eq:effective_logits} with first order gradient method.
    Solid lines show a correct teacher \(y_{\mathrm{g}} = \mathds{1}\{ z_{\mathrm{t}} > 0 \} \), and dashed lines denote a wrong teacher \(y_{\mathrm{g}} = \mathds{1}\{ z_{\mathrm{t}} < 0 \} \). The existence of hard label produces a discontinuity in  \(z_{\mathrm{s,eff}}(z_{\mathrm{t}}, y_{\mathrm{g}})\).
  }
  \vspace{-0.2cm}
\end{figure}

We consider a binary classification problem on \(x\in \mathcal{X} \subseteq \mathbb{R}^d\). We assume \(x\) has a fixed distribution \(P(x)\) and a ground truth classification boundary , \(y = \mathds{1}\{f_{\mathrm{g}}(x)>0\}\in \{0, 1\}\). The task is to find a student neural network \(z = f(x;w): \mathcal{X} \mapsto \mathbb{R}\) that best fits the data. \(w\) are network's weights and  \(z\) is its output logit. The parameters are initialized with NTK-parameterization in \cite{jacot2018neural},
\begin{equation}
  \begin{aligned}
    &h^{1} = \sigma_W W^{0}x/\sqrt{d} + \sigma_b b^{0}, x^{1}=g(h^{1}),\\
    &h^{l+1} = \sigma_W W^{l}x^{l}/\sqrt{m} + \sigma_b b^{l}, \;
    x^{l+1} = g(h^{l+1}),\;
    l=1,2,\cdots, L-1, \\
    &f(x; w) = \sigma_W W^{L}x^{L}/\sqrt{m} + \sigma_b b^{L},\; w = \cup_{i=0}^{L}\{W^{i}, b^{i}\},\;
  \end{aligned}
\end{equation}
where \(g(\cdot)\) is the activation function, \(W^{l}_{ij},b^{l}_i\sim \mathcal{N}(0,1)\) are the parameters of each layer, and \((\sigma_W,\sigma_b)\) are the hyperparameters. \((\sigma_W,\sigma_b)\) are choose to be \((1.0,1.0)\) throughout our experiments. The network trained by empirical risk minimization \(\hat{w} ={\arg \min }_{w} \mathcal{L}\) with (stochastic) gradient descent, and the \textit{distillation loss} is
\begin{equation}
  \mathcal{L} = \frac{1}{N}\sum_{n=1}^{N}\ell_{n} = \frac{1}{N}\sum_{n=1}^{N} \rho H(y_{\mathrm{t},n}, \sigma(\frac{z_{\mathrm{s},n}}{T})) + (1-\rho)H(y_{\mathrm{g},n}, \sigma(z_{\mathrm{s},n})) ,
\end{equation}
where \(H(p,q) = -\left[p\log q + (1-p) \log(1-q)\right]\) is binary cross-entropy loss, \(\sigma(z) = 1/(1+\exp(-z))\) is sigmoid function, \(z_{\mathrm{s},n} = f(x_n;w)\) are the output logits of \textit{student} network, \(y_{\mathrm{t},n} = \sigma(z_{\mathrm{t},n}/T)\) are soft labels  of \textit{teacher} network and \(y_{\mathrm{g},n} = \mathds{1}\{f_{\mathrm{g}}(x_n)>0\} \) are the \textit{ground truth} hard labels. \(T\) is \textit{temperature} and \(\rho\) is \textit{soft ratio}.

In this paper we focus on student's convergence behavior and neglect its training details. \emph{We assume the student is over-parameterized and wide}. 
\cite{du2019gradient} proves the convergence of network to global minima for L2 loss under this assumption. We believe this holds true for distillation loss, and give a further discussion in Appendix (Sec.~\ref{supp:distillation_loss_convergence}).
This means that as training time \(\tau\to\infty\), each student logit would converge to a target value that minimizes the loss of each sample,
\begin{equation}\label{eq:effective_logits}
  \lim_{\tau\to\infty} z_{\mathrm{s}}(\tau) = \hat{z}_{\mathrm{s}}, \quad \frac{d \ell}{d \hat{z}_{\mathrm{s}}} = \frac{\rho}{T}(\sigma(\hat{z}_{\mathrm{s}}/T) - \sigma(z_{\mathrm{t}}/T)) + (1-\rho)(\sigma(\hat{z}_{\mathrm{s}}) - y_{\mathrm{g}}) = 0.
\end{equation}
The implicit solution to this equation defines an \textit{effective student logit} \(z_{\mathrm{s,eff}} = \hat{z}_{\mathrm{s}}\) as a function of \(z_{\mathrm{t}}\) and \(y_{\mathrm{g}}\). In Fig.~\ref{fig:1} we plot \(z_{\mathrm{s,eff}}(z_{\mathrm{t}}, y_{\mathrm{g}})\) with varying soft ratio \(\rho\) and temperature \(T\). Solid curves show \(z_{\mathrm{s,eff}}\) solved by Eq.~\ref{eq:effective_logits} of a correct teacher \(y_{\mathrm{g}} = \mathds{1}\{ z_{\mathrm{t}} > 0 \}\), and dashed curves denote a wrong teacher.
We can observe that the presence of hard label creates a split in the shape of \(z_{\mathrm{s,eff}}\), and this split increases as hard label takes higher ratio. The generalization and optimization effect of this split will be discussed in Sec.~\ref{sec:transfer_risk_bound} and \ref{sec:data_inefficiency}.

The temperature $T$ is introduced by Hilton in \cite{hinton2015distilling} to soften teacher's probability. Intuitively \(\sigma'(z) \to 0 \) when \(\sigma(z)\to\{0,1\}\), so a higher \(T\) makes student easily converged.
However, this is only a benefit during training. In Fig.~\ref{fig:1} we show that converged student logits are always bigger than their teachers', \(|z_{\mathrm{s,eff}}|>|z_{\mathrm{t}}|\). We also observed that when \(T>1\), a higher temperature causes the split to be wider, and when \(T<1\), the curve forms a \(T\)-independent plateau. In this paper we are interested in the effect of soft ratio \(\rho\). Therefore, we set temperature to be fixed and follow the convention \(T>1\).

Wide network assumption also allows us to linearize student network with NTK techniques. According to \cite{jacot2018neural,lee2019wide}, the outputs of an infinitely wide network are approximately its linearization at initialization, \(f(x;w_{\mathrm{nlin}}) \approx f(x;w_0) + \Delta_{w}^{\top}\phi(x),\) where \(\phi(x) = \partial_{w}f(x;w_0)\in \mathbb{R}^{ p}\) are called \textit{random features}, \(p\) is the dimension of weight, \(w_{\mathrm{nlin}}\) is the weight of original network  and \(\Delta_{w} = w - w_0 \in \mathbb{R}^p\) is the \textit{weight change} of linear model trained by same strategy as the original network.
The converged weight change \(\Delta_{\hat{w}}\) is used throughout this paper. For training samples of \(\mathbf{X} = [x_1, \cdots, x_n] \in \mathbb{R}^{d\times n}\) and \(\mathbf{z} = [z_1, \cdots, z_n]^{\top}\in \mathbb{R}^{n\times 1}\), the converged weight change is
\begin{equation}
  \label{eq:learned_weight}
  \Delta_{\hat{w}} = \phi(\mathbf{X})(\hat{\Theta}(\mathbf{X},\mathbf{X}))^{-1}\Delta_{\mathbf{z}},
\end{equation}
where \(\hat{\Theta}(\mathbf{X},\mathbf{X})  = \phi(\mathbf{X})^{\mathsf{T}} \phi(\mathbf{X})\) is called the empirical kernel, \(\Delta_{\mathbf{z}} = \mathbf{z} - \mathbf{z}_0\) and \(\mathbf{z}_0 = f(\mathbf{X};w_0)\) are student's logits at initialization.
According to \cite{jacot2018neural,lee2019wide}, in infinite width limit, \(\hat{\Theta}(\mathbf{X},\mathbf{X})\) converges to a weight independent kernel \(\Theta(\mathbf{X},\mathbf{X})\), called NTK.
We assume \(\hat{\Theta}(\mathbf{X},\mathbf{X}) \approx \Theta(\mathbf{X},\mathbf{X})\) throughout this paper and simplify it as \(\Theta_n\).
Eq.~\ref{eq:learned_weight} can be proved by Theorem 1 of \cite{phuong2019towards}. Eq.8 of \cite{lee2019wide} also gives a similar result of \(\ell_2\)-loss. Note that due to the extremely high dimension of \(\phi(\mathbf{X})\), direct calculation of Eq.~\ref{eq:learned_weight} is impractical. We can instead attain \(\Delta_{\hat{w}}\) by training the linear model.

The rest of this paper are organized as follows: In Sec.~\ref{sec:transfer_risk_bound} we give our transfer risk bound of linearized network. This bound is computationally expensive, so in Sec.~\ref{sec:data_inefficiency} we introduce a more direct metric, called data inefficiency. Then we analyze the effect of teacher's early stopping and soft ratio with this metric. Sec.~\ref{sec:transfer_risk_bound} and \ref{sec:data_inefficiency} only consider a perfect teacher and under this setting, hard labels are not beneficial for KD.
Therefore as a complementary, we study the effect of hard labels in imperfect teacher distillation in Sec.~\ref{sec:imperfect}.

\section{\label{sec:transfer_risk_bound}Transfer Risk Bound}

\begin{figure}[htbp]
  \centering
  \includegraphics[width=0.32\textwidth]{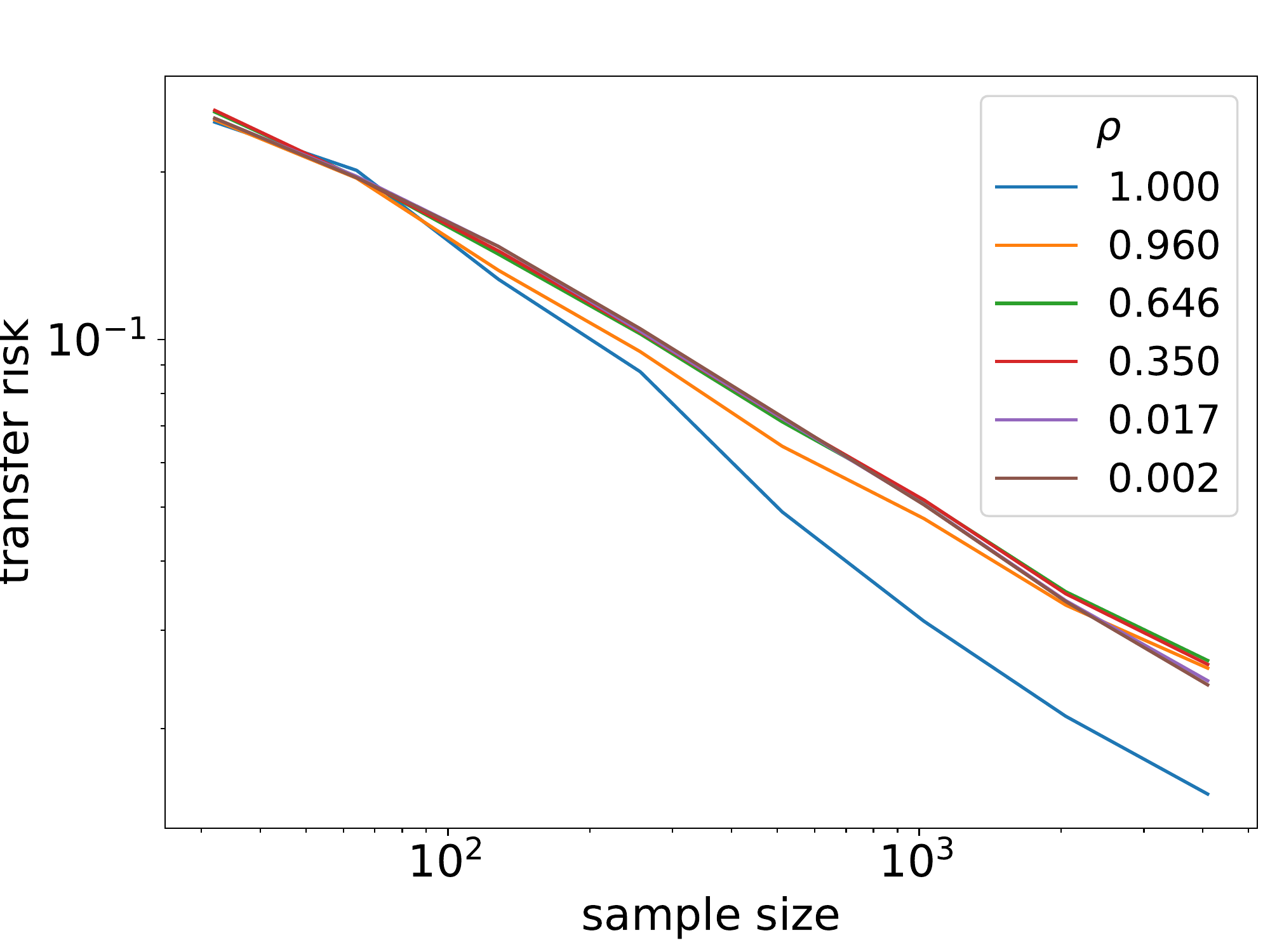}
  \includegraphics[width=0.32\textwidth]{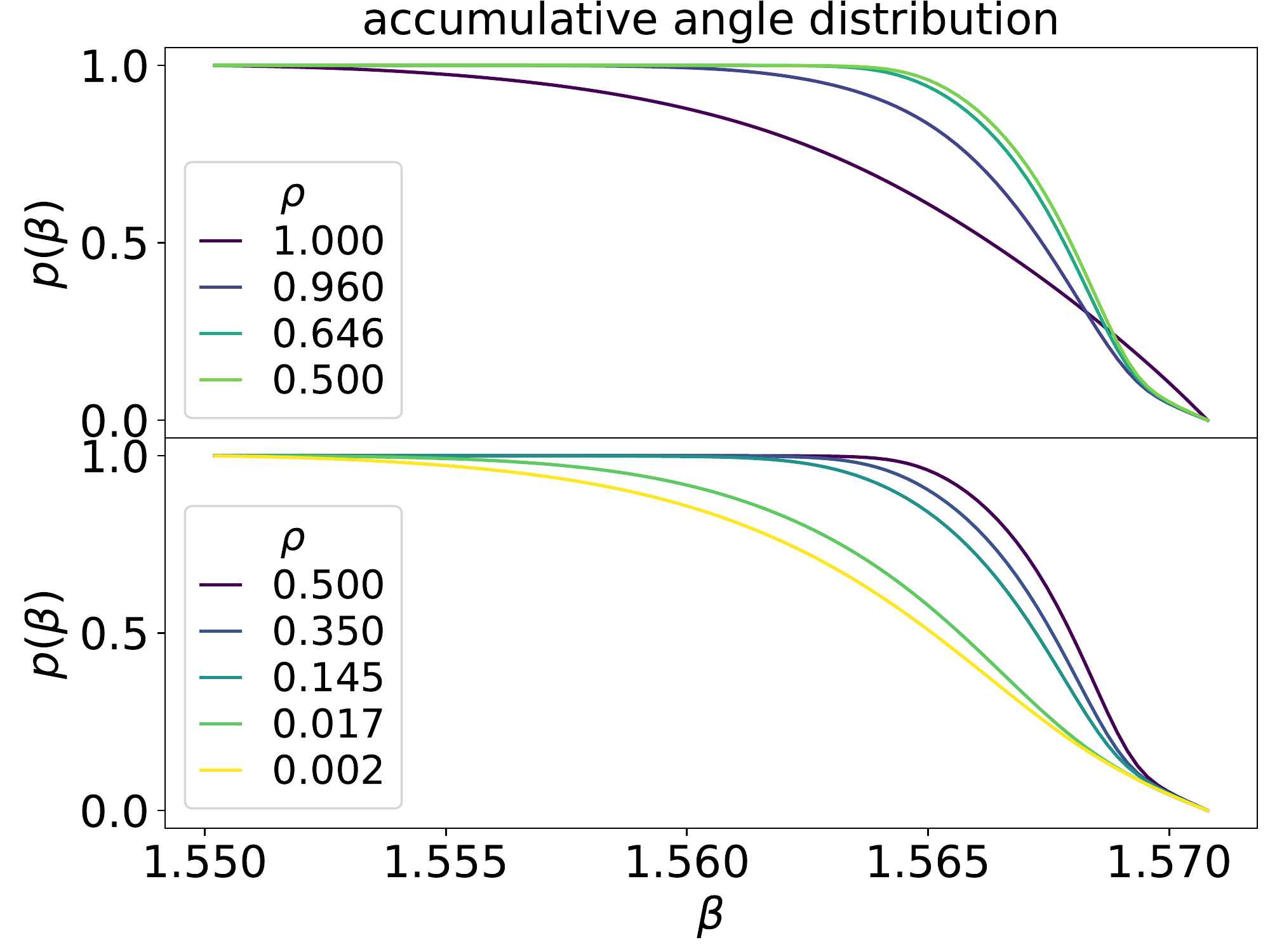}
  \includegraphics[width=0.32\textwidth]{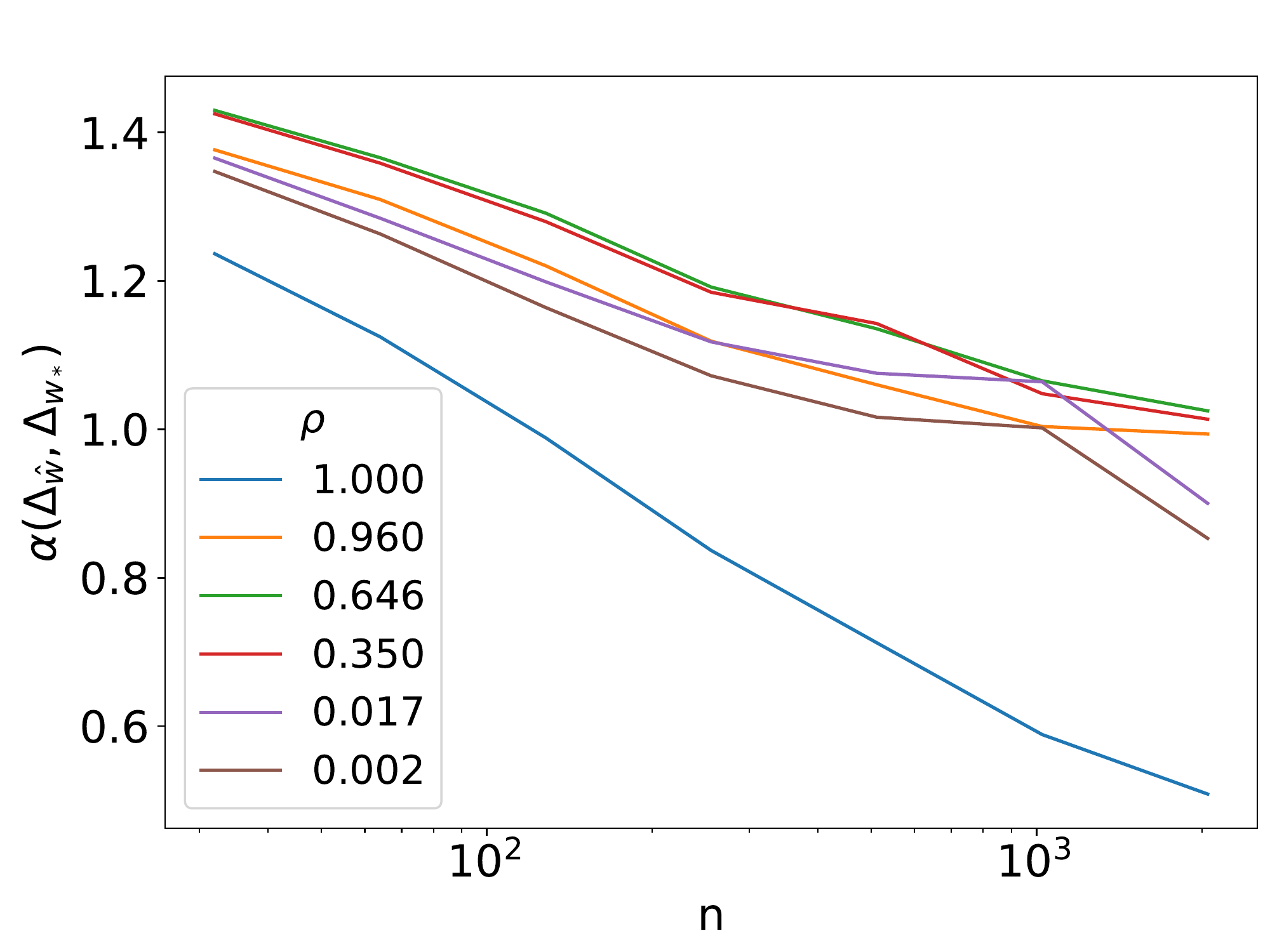}
  \vspace{-0.2cm}
  \caption{\label{fig:2}  \textbf{Left:} Experimental transfer risk, plotted with respect to sample size \(n\). The curve shows a power law relation, with a faster rate for pure soft distillation. \textbf{Middle:} Accumulative angle distribution \(p(\beta)\), as part of our transfer risk bound. We split the curves into two subfigures because they change non-monotonically with respect to \(\rho\). \textbf{Right:}
  \(\alpha_n'=\alpha (\Delta_{\hat{w}},\Delta_{w_*})\) with respect to \(n\). See Sec.~\ref{supp:experiment_detail} in Appendix for details.}
  \vspace{-0.2cm}
\end{figure}

The \textit{transfer risk} is the probability of different prediction from teacher, \(\mathcal{R} =\underset{x\sim P(x)}{\mathbb{P}} [z_{\mathrm{t}}\cdot z_{\mathrm{s}} < 0]\). Before we state our bound, we first present our observation on perfect teacher distillation (Fig.~\ref{fig:2} left).  \(\mathcal{R}\) shows a power relation with respect to sample size \(n\), and pure soft distillation shows a significantly faster converging rate. This power law is also observed in practical distillation of ResNet on CIFAR (see Sec.~\ref{supp:Resnet_vary_datanum} in Appendix). 
The relation of \(\mathcal{R}\) and \(n\) is rarely discussed in deep learning. However in statistical learning, \cite{tsybakov2004optimal,massart2006risk} prove that given some criterion on the logits distribution near class boundary, this rate can approach \(O(1/n)\), instead of  \(O(1/\sqrt{n})\) given by traditional VC theory. Our results show experimental connections to their claim.
We try to explain this with our transfer risk bound. Compared with the bound in \cite{phuong2019towards} for linear distillation, our bound is a tightened version of their results and a modification to linearized network, as elaborated in the following.

We assume the student is expressive enough to approximate effective logits and zero function.
Then \(z_{\mathrm{s,eff}} \approx f(x;w_0) + \Delta_{w_*}^{\top}\phi(x)\) and \(0 \approx f(x;w_0) + \Delta_{w_{\mathrm{z}}}^{\top}\phi(x)\) where \(\Delta_{w_*}\) is \textit{oracle weight change}  and \(\Delta_{w_{\mathrm{z}}}\) is \textit{zero weight change}. This approximation relies on the expressiveness of student, and has no dependency on teacher's weight. The student can be simplified as a bias-free linear model, \(z \approx (\Delta_{w} - \Delta_{w_{\mathrm{z}}})^{\top}\phi(x)\) and transfer risk is \(\mathcal{R}= \underset{x\sim P(x)}{\mathbb{P}}[ (\Delta_{w_*} - \Delta_{w_{\mathrm{z}}})^{\top}\phi(x)\cdot (\Delta_{\hat{w}} - \Delta_{w_{\mathrm{z}}})^{\top}\phi(x)<0 ]\).
Further, student's weight change can be written as a projection onto \(\mathrm{span}\{\phi(\mathbf{X})\}\),
\begin{equation}
  \label{eq:projection}
  \Delta_{\hat{w}} = \phi(\mathbf{X})\Theta_n^{-1}\phi(\mathbf{X})^{\top}\Delta_{w_*} = \mathbf{P}_{\Phi} \Delta_{w_*},
\end{equation}
where \(\mathbf{P}_{\Phi}\) is the projection matrix. We denote \(\bar{\alpha}(a,b) = \cos^{-1}(|x^\top y| / \sqrt{x^\top x \cdot y^\top y}) \in [0,\pi/2]\) as the angle of two vectors \(a,b\). Acute angle is used so that wrong prediction of both classes is counted. Similar to \cite{phuong2019towards}, the \textit{accumulative angle distribution} \(p(\beta)\) of given distribution is defined as
\begin{equation}
  p(\beta) = \underset{x\sim P(x)}{\mathbb{P}}\left[\bar{\alpha}(\phi(x), \Delta_{w_*} - \Delta_{w_{\mathrm{z}}}) > \beta\right],\;\mathrm{for}\;\beta\in[0,\pi/2].
\end{equation}
Now we state our result of transfer risk. The proof is in Sec.~\ref{supp:proof_of_risk_bound} of Appendix.
\begin{theorem} \label{theorem:bound}
  \textbf{(Risk bound)} Given input distribution \(P(x)\), training samples \(\mathbf{X} = [x_1, \cdots, x_n]\), oracle weight change \(\Delta_{w_*}\), zero weight change \(\Delta_{w_{\mathrm{z}}}\) and accumulative angle distribution \(p(\beta)\), the transfer risk is bounded by,
  \begin{equation}
    \mathcal{R}_n \leq p(\frac{\pi}{2} - \bar{\alpha}_n),
  \end{equation}
  where \(\bar{\alpha}_n = \bar{\alpha}(\Delta_{w_*} - \Delta_{w_{\mathrm{z}}}, \Delta_{\hat{w}} - \Delta_{w_{\mathrm{z}}})\) and \(\Delta_{\hat{w}}\) is student's converged weight change.
\end{theorem}

\subsection{Estimation of the Bound}

The \(p(\beta)\) and \(\bar{\alpha}_n\) in Theorem.~\ref{theorem:bound} can be estimated with NTK techniques. For \(p(\beta)\), calculation involves sampling over \(\cos \bar{\alpha}(\phi(x), \Delta_{w_*} - \Delta_{w_{\mathrm{z}}})\), and
\begin{equation} \label{eq:cos_approx}
  \cos \bar{\alpha}(\phi(x), \Delta_{w_*} - \Delta_{w_{\mathrm{z}}}) = \frac{( \Delta_{w_*} - \Delta_{w_{\mathrm{z}}})^T\phi(x)}{|| \Delta_{w_*} - \Delta_{w_{\mathrm{z}}}||_2\cdot ||\phi(x)||_2} = \frac{z_{\mathrm{s,eff}}(x)}{|| \Delta_{w_*} - \Delta_{w_{\mathrm{z}}}||_2\cdot \sqrt{\Theta(x,x)}}.
\end{equation}
\(\Delta_{w_*}\) and \(\Delta_{w_\mathrm{z}}\) can be approximated by online-batch SGD training, which is equivalent to training with infinite amount of samples. Fig.~\ref{fig:2} middle shows estimations of \(p(\beta)\) of this method. It shows a non-monotonicity with \(\rho\), but all curves shows a near linear relation with \(\beta\) near \(\beta\to \pi/2\).

For \(\bar{\alpha}_n\), training an actual student is approachable, but we can also approximate it beforehand with a straightforward view. Usually, zero function is much easier to converge than a normal function, and \(||\Delta_{w_{\mathrm{z}}}||_2\ll ||\Delta_{w_*}||_2\) (see Sec.~\ref{supp:zero_function} in Appendix). Then,
\begin{equation}
  \cos \bar{\alpha}_n \approx
  \cos \bar{\alpha}(\Delta_{w_*}, \Delta_{\hat{w}}) = \frac{\Delta_{w_*}^{\top}\mathbf{P}_{\mathbf{\Phi}} \Delta_{w_*}}{||\Delta_{w_*}||_2||\mathbf{P}_{\mathbf{\Phi}} \Delta_{w_*}||_2} = \frac{\sqrt{\Delta_{\mathbf{z}}^{\top} \Theta_n^{-1}\Delta_{\mathbf{z}}}}{||\Delta_{w_*}||_2} = \frac{||\Delta_{\hat{w}}||_2}{||\Delta_{w_*}||_2}.
\end{equation}
Take \(\rho=1\) as an example. From Fig.~\ref{fig:4} right we empirically observe that \(\partial \ln ||\Delta_{\hat{w}}||_2 / \partial \ln n \sim n^{-b}\), where \(b\) is the parameter to describe this relation. By integrating over \(n\), we can get an asymptotic behavior of \(\alpha_n \sim n^{-b/2}\). Then based on the near linearity of \(p(\beta)\) near \(\beta\to \pi/2 \) our result gives a bound of \(O(n^{-b/2})\).
When \(b>1\), this bound outperforms classical bounds. However, we are not certain whether this is the case since \(b\) depends on various hyper-parameters, but we do ﬁnd \(b\) to be bigger when teacher’s stopping epoch is small (i.e. the task is easy).
Note that this estimation requires the existence of a converging \(||\Delta_{\hat{w}}||_2\) with respect to \(n\). These assumptions fail to be satisfied when hard labels are added to the loss (Fig.~\ref{fig:4} right). 
As a complement, in Fig.~\ref{fig:2} right we plot the direct calculation of \(\alpha_n\) of different soft ratio. The result shows the fastest convergence speed in pure soft distillation case. This explains the result of Fig.~\ref{fig:2} left about the fast convergence with pure soft labels.

\subsection{Tightness Analysis}
First we show that the risk bound in \cite{phuong2019towards} is quite loose in high dimensions. Both their results and ours use a property that \(\bar{\alpha}'_n = \alpha(\Delta_{\hat{w}}, \Delta_{w_*})\) monotonically decreases with respect to sample size \(n\) (Lemma.1 in \cite{phuong2019towards}).
However, they utilize this property loosely that \(\bar{\alpha}'_n\) is approximated by \(\alpha\) trained by one sample, \(\bar{\alpha}'_n \leq  \bar{\alpha}'_1 = \bar{\alpha}(\Delta_{w_*}, \phi(x_i))\).
This leads to a risk bound of \(R_n \leq {\min}_\beta\; p(\beta) + p(\pi/2 - \beta)^n\).
Due to the high dimension of random vector  \(\phi(x)\), \(\phi(x)\) and \(\Delta_{w_*} - \Delta_{w_{\mathrm{z}}}\) are very likely to be perpendicular (Fig.~\ref{fig:2} middle).
We can further show (see Sec.~\ref{supp:angle_bound} in Appendix) that for ReLU activation, \(p(\beta) \equiv 1\) strictly for \(\beta \in [0,\beta_t],\beta_t\approx \pi/2\) and therefore their bound is strictly \(\mathcal{R} \equiv 1\).
We tighten their result by directly using \(\alpha'_n\) in risk estimation. The improvement is significant since even if \(\bar{\alpha}'_1 \approx \pi/2\), \(\bar{\alpha}'_n\) can be small when \(n\) is sufficiently large.
However, we have to point out that our bound also shows certain looseness with small sample size due to the fact that \(p(\pi/2 - \alpha_{n_{\mathrm{small}}}) \approx 1\).
The generalization ability of small sample size remains mysterious.

Our approach and \cite{phuong2019towards}  differ from classic learning theory for generalization bound of neural network~\cite{neyshabur2019role,cao2019generalization2,arora2019fine}.
They are based on Rademacher complexity \(\mathfrak{R}\) of hypothesis space \(\mathcal{H}\), and give a bound like \(L_{\mathcal{D},\mathrm{0-1}}(f) \leq 2 \mathfrak{R}_{n}(\ell \circ \mathcal{H})+ O(\sqrt{\ln (1 / \delta)/n })\)
where \(\mathfrak{R}_{n}(\mathcal{H})= {\mathbb{E}}_{\xi \sim\{\pm 1\}^{n}}\left[\sup _{f \in \mathcal{H}} \sum_{i=1}^{n} \xi_{i} f\left(x_{i}\right)\right] / n\).
A common way to tighten this bound is to restrict \(\mathcal{H}\) near the network's initialization, \(\mathcal{H} = \{f(x;w)| w\in \mathbb{R}^p, \mathrm{s.t.} ||w - w_0||_2 \leq B\}\) (\cite{arora2019fine,cao2019generalization,cao2019generalization2}).
However as \cite{zhang2016understanding} shows, the generalization error is highly task related.
Even if \(\mathcal{H}\) is restricted, Rademacher complexity still only captures the generalization ability of the hardest function in \(\mathcal{H}\).
Our approach, instead, tackles directly the the task related network function \(f(x) = f(x;w_0) + \Delta_{w_*}^{\top}\phi(x)\).
Therefore the hypothesis space of our bound is much smaller.

\section{\label{sec:data_inefficiency} Data Inefficiency Analysis}

In the discussion above, \(||\Delta_{\hat{w}}||_2\) plays an important role in the calculation of \(p(\beta)\) and angle convergence. However, \(p(\beta)\) needs much effort to calculate and cannot show the obvious advantage of soft labels.
In this section, we define a metric on the task's training difficulty, called \textit{data inefficiency} to directly measure the change of \(||\Delta_{\hat{w}}||_2\). We first state its rigorous definition and then discuss how the stopping epoch of teacher and soft ratio  affect data inefficiency.

\begin{definition}
\textbf{(Data Inefficiency)}
We introduce data inefficiency as a discrete form of \(\partial \ln ||\Delta_{\hat{w},n}||_2 / \partial \ln n\),
\begin{eqnarray}\label{eq:inefficiency}
  \mathcal{I}(n) = n \left[\ln  \mathbb{E}||\Delta_{\hat{w},n+1}||_2 - \ln  \mathbb{E} ||\Delta_{\hat{w},n}||_2\right]
\end{eqnarray}
where \(||\Delta_{\hat{w},n}||_2 = \sqrt{\Delta_{\mathbf{z}_{n}}^{\top}\Theta_{n}^{-1} \Delta_{\mathbf{z}_{n}}}\) is the norm of student's converged weight change trained by \(n\) samples.
\end{definition}
The expectation is taken over the randomness of samples and student's initialization. Logarithm is used to normalize the scale of \(\Delta_z\), and to reveal the power relation of  \(||\Delta_{\hat{w}}(n)||_2\). We define data inefficiency as Eq.~\ref{eq:inefficiency} for the reasons using the following principle.
\begin{principle}
\(\mathcal{I}(n)\) reveals the student's difficulty of weight recovery from teacher.
\end{principle}
For better explanation, we assume again there exists an oracle weight change \(\Delta_{w_*}\) as we did in Sec.~\ref{sec:transfer_risk_bound}. Then student's weight change is a projection, \(\Delta_{\hat{w}} =  \mathbf{P}_{\Phi} \Delta_{w_{*}}\), where \(\mathbf{P}_{\Phi}\) is a projection matrix onto \(\mathrm{span}\{\phi(\mathbf{X})\}\). As \(n\) increases and \(\mathrm{span}\{\phi(\mathbf{X})\}\) expands, \(\Delta_{\hat{w}}\) gradually recovers to \(\Delta_{w_*}\), and \(||\Delta_{\hat{w}}||_2 = \sqrt{\Delta_{\hat{w}}^\top \mathbf{P}_{\Phi} \Delta_{\hat{w}}}\) shows the stage of recovery.
If the task is data efficient, we can recover the main component of \(\Delta_{w_*}\) with a small sample size, and further we expect \(||\Delta_{\hat{w}}||_2\) not to increase very much with \(n\).
Reversely, if the task is inefficient, the same sample set is not sufficient to recover the main component of \(\Delta_{w_*}\), so we expect \(||\Delta_{\hat{w}}||_2\) to continue increasing.
Therefore, we use \(\mathcal{I}(n)\) to indicate the  increasing speed of \(||\Delta_{\hat{w},n}||_2\) with respect to \(n\) and a faster increasing (or slower converging) \(\mathcal{I}(n)\) indicates a less data efficient task to train.

To demonstrate this principle, we perform two difficulty control tasks of learning Gaussian mixture function 
\(z_{\mathrm{gaussian}}(x) = \sum_{j=1}^{q} A_j \exp(-(x-x_j)^2/\sigma_j^2)\) (see Sec.~\ref{supp:experiment_detail} in Appendix for details).
The difficulty of the first task is controlled by the number of Gaussian modes \(q\) (Fig.~\ref{fig:3} left).
%
In the second task (Fig.~\ref{fig:3} right), we control difficulty by the probability of sign flipping \(z_* = s\times z_{\mathrm{gaussian}}(x)\), where \(s\) has probability of \(\{1-p_{\mathrm{flip}}, p_{\mathrm{flip}}\}\) to be \(\{1, -1\}\).
Both experiments show that \(\mathcal{I}(n)\) ranges as the order of difficulties, which agrees with our idea on data inefficiency.

\begin{figure}[t!]
  \centering
  \includegraphics[width=0.32\textwidth]{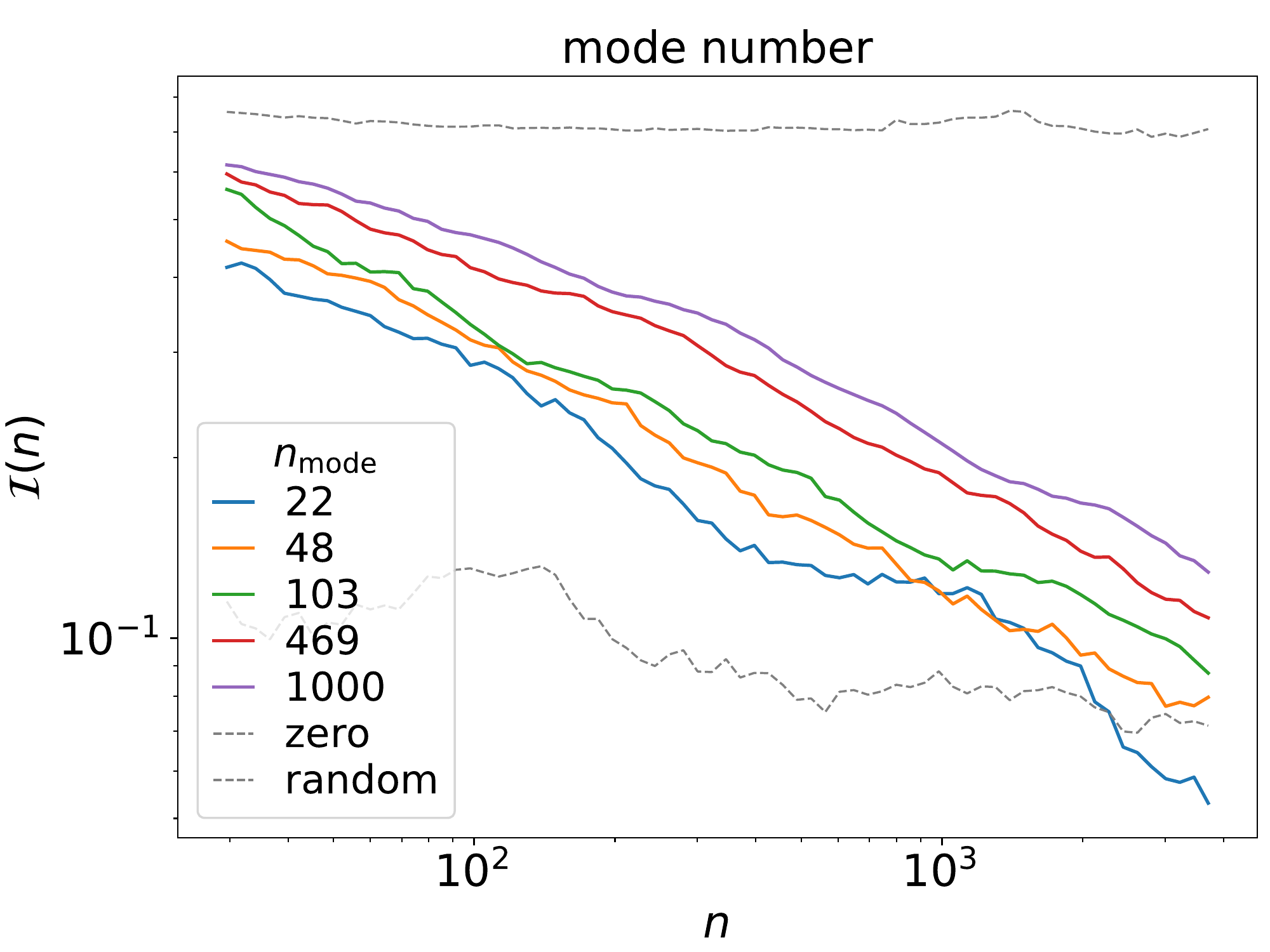}
  \includegraphics[width=0.32\textwidth]{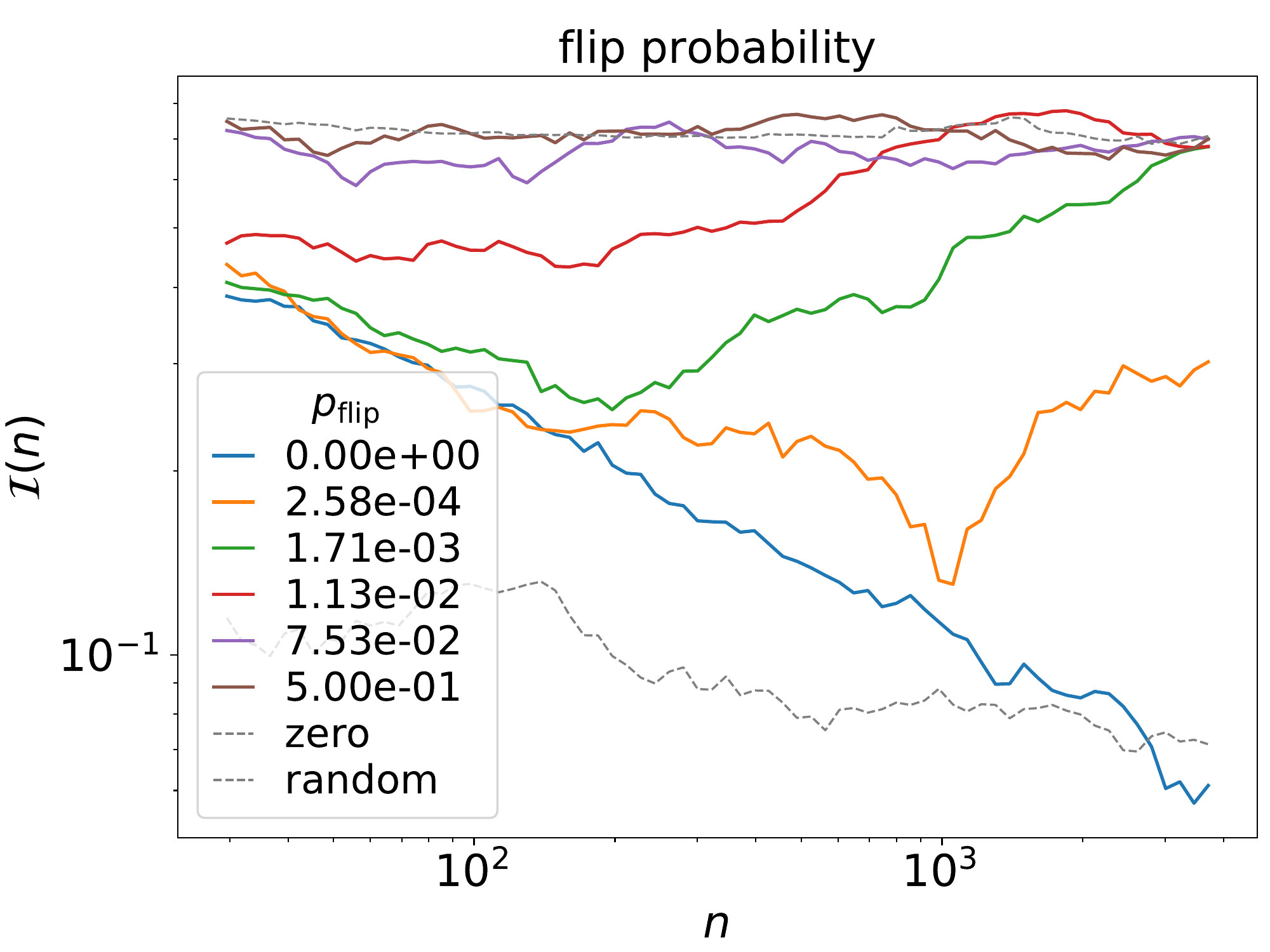}
  \vspace{-0.2cm}
  \caption{\label{fig:3}
    \textbf{Left:} Difficulty control on the number of modes. The figure shows \(\mathcal{I}(n)\) of learning different Gaussian mixture function.
    The decreasing behavior of \(\mathcal{I}(n)\) is typical for learning a noise-free smooth function.
    \textbf{Right:} Difficulty control on flip probability. The figure plots \(\mathcal{I}(n)\)  of learning the same function with different noise level. \(p_{\mathrm{flip}}=0.5\) means a completely random sign.
    The noise makes these tasks so difficult to learn that \(\mathcal{I}(n)\equiv 0.8\), this means \(\Delta_{\hat{w}}\) will not converge. The two figures demonstrate a positive correlation between \(\mathcal{I}(n)\) and task's difficulty.
    The dashed lines are references of a hard and easy task. The upper dashed line shows the complexity of random label \(\Delta_{z} \sim \mathcal{N}(0,1)\), while the lower dashed line shows the complexity of zero function \(z \equiv 0\). The later one also demonstrates that zero function is extremely easy to learn and \(\Delta_{w_{\mathrm{z}}}\) can be neglected. All the results are based on the average of 20 runs.
  }
    \vspace{-0.2cm}
\end{figure}

Our idea is also supported by other works. \cite{arora2019fine} proves that for 2-layer over-parameterized network trained with \(\ell_2\)-loss, \(\sqrt{\Delta_{\mathbf{z}}^{\top}\Theta_n^{-1} \Delta_{\mathbf{z}} / n}\) is a generalization bound for the global minima. \cite{cao2019generalization} studies \(L\)-layer over-parameterized network trained by online-SGD and get a similar bound of \(\widetilde{\mathcal{O}}[L \cdot \sqrt{\Delta_{\mathbf{z}}^{\top}\Theta_n^{-1} \Delta_{\mathbf{z}} / n}]\).
A slower increasing \(\sqrt{\Delta_\mathbf{z}^{\top}\Theta_n^{-1} \Delta_\mathbf{z}}\) means a faster decreasing generalization error, which also means the task is more data efficient.
\cite{ronen2019convergence} supports our idea from the perspective of optimization. They calculate \(\Theta(x,x)\)'s eigensystem in the case of uniform distribution on \(S^d\) and find that convergence time of learning the \(i\)th eigen function with \(\ell_2\)-loss is proportional to \(\lambda_i(\Theta)^{-1} \propto \mathbf{z}^{\top}\Theta_n^{-1} \mathbf{z}\).

\subsection{Early Stopping Improves Data Efficiency}

\begin{figure}[ht]
  \centering
  \includegraphics[width=0.32\textwidth]{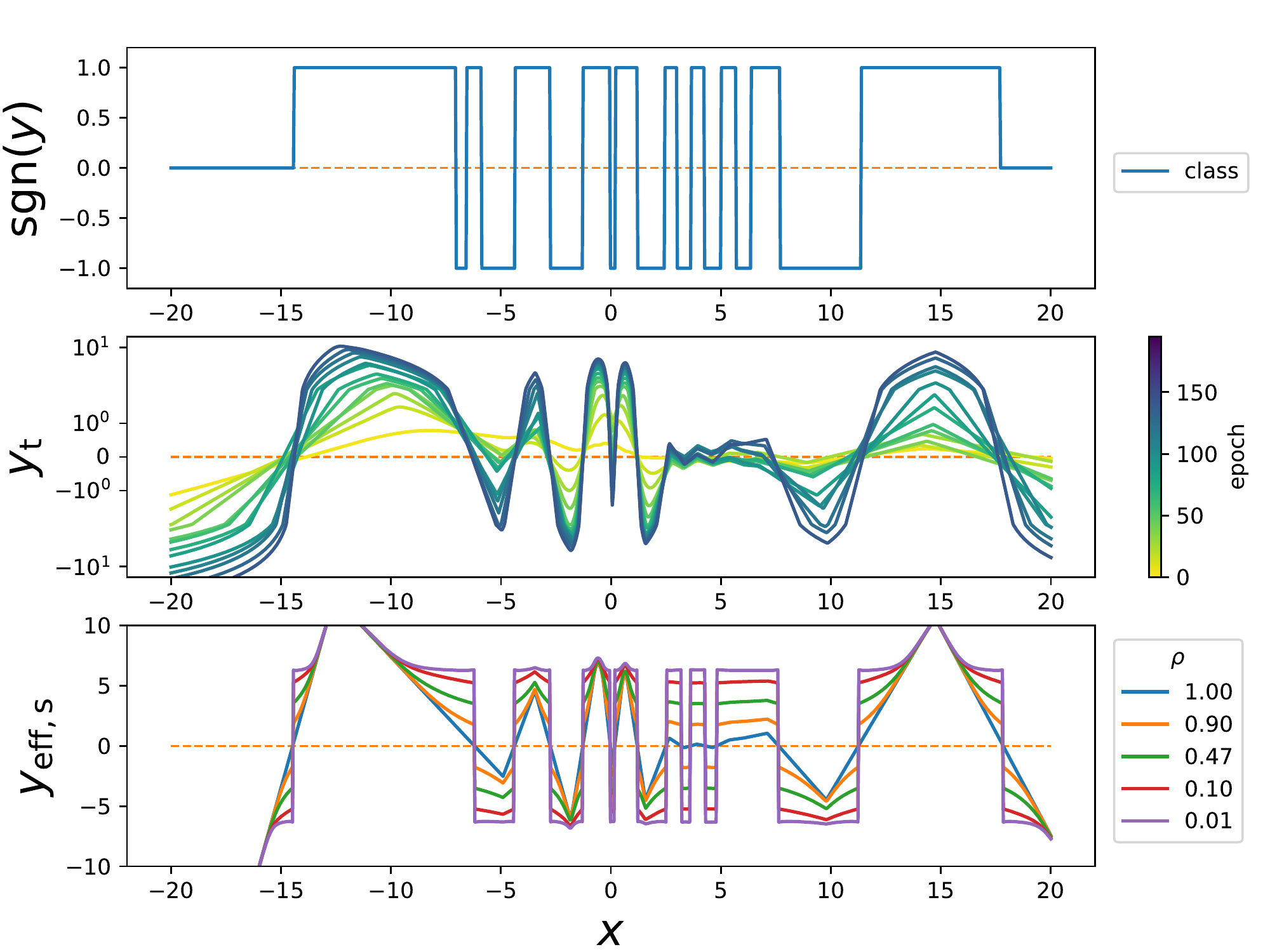}
  \includegraphics[width=0.32\textwidth]{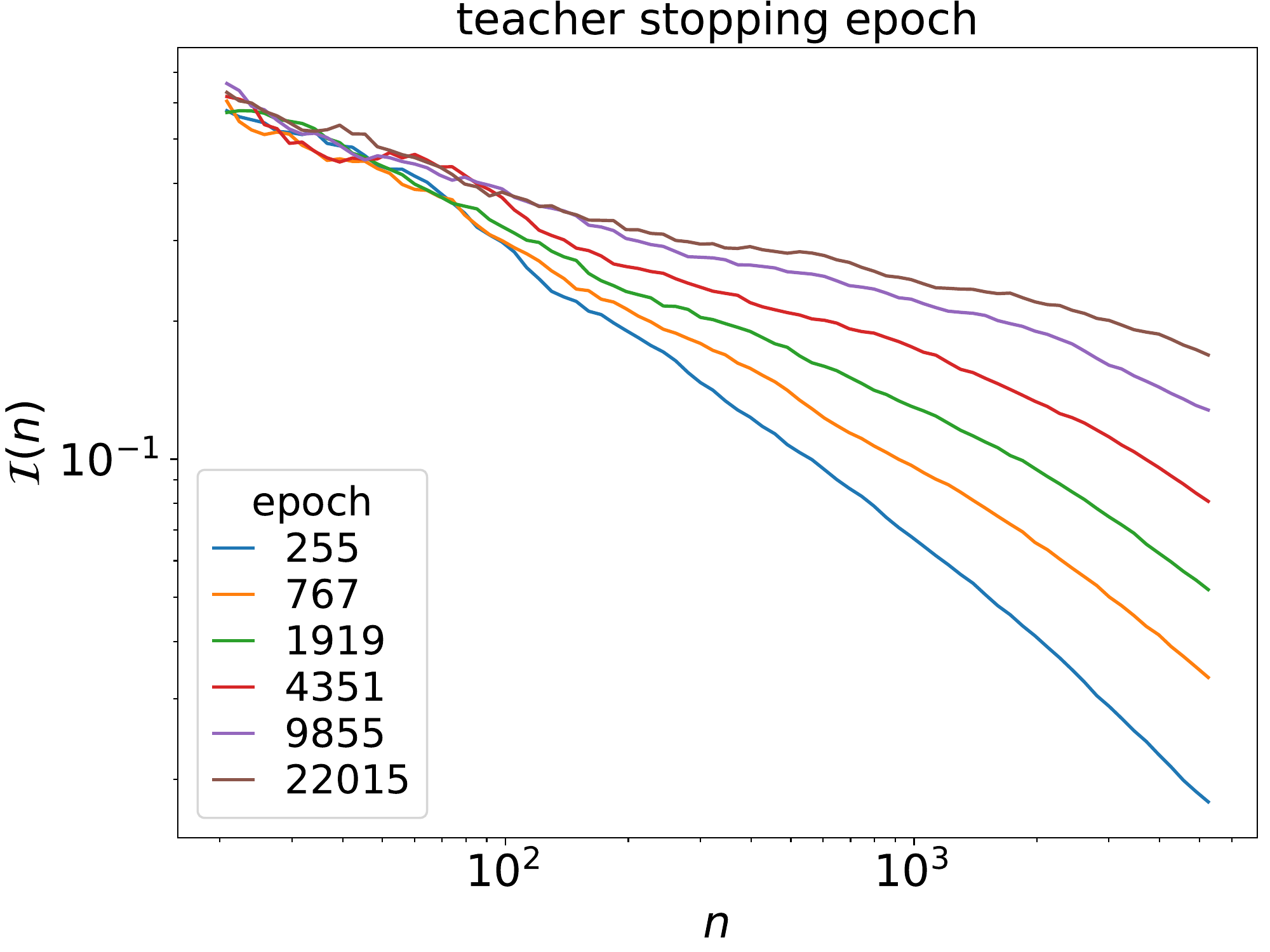}
  \includegraphics[width=0.32\textwidth]{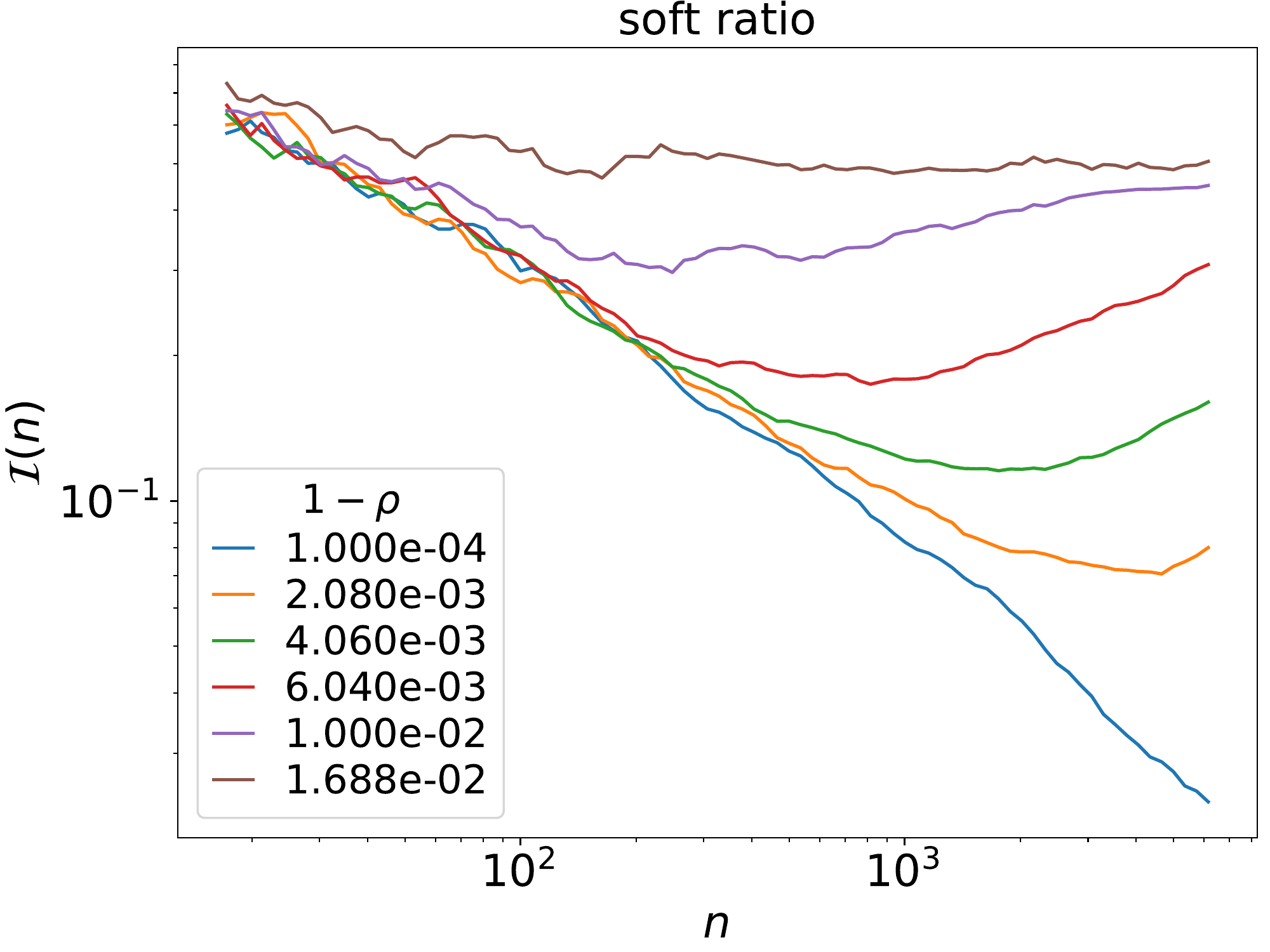}
  \vspace{-0.2cm}
  \caption{
    \label{fig:4}
    \textbf{Left:} A 1-D example of teacher and student output. \textbf{Left Top}: Ground truth class boundary. \textbf{Left Middle}: Teacher's logits at different stopping epoch. The scale of teacher increases and shape becomes more detailed while training. \textbf{Left Bottom}:  Effective student logits \(y_{\mathrm{s,eff}}\) at different soft ratio \(\rho\). This figure further illustrates the discontinuity in \(y_{\mathrm{s,eff}}\). For a small \(\rho\), students shape shows a clear similarity with label smoothing. All student share a same teacher network. 
    \textbf{Middle and Right:} Data inefficiency curve of different teacher stopping epoch and soft ratio.
    It shows that adding hard labels to distillation increases sample complexity. 
    See Sec.~\ref{supp:experiment_detail} for details.
  }
  \vspace{-0.2cm}
\end{figure}

It is a well-known result that neural network first learns the low frequency components of the target function~\cite{xu2019frequency}. Therefore, early stopping the teacher provides the student with a simplified version of the target function, which may be easier to train. For the integrity of this work, we show in Fig.~\ref{fig:4} middle that early stopping improves the data efficiency.

\cite{cho2019efficacy} experimentally observes that a more accurate teacher may not teach a better student, and early stopping the teacher helps the student to fit the data better.  A recent work \cite{dong2019distillation}  also emphasizes the importance of early stopping in KD theoretically. They assume a probability of data corruption and treat the high frequency component in teacher's label as a result of noise. They claim that early stopping prevents teacher from learning the noise in its label.

\subsection{\label{subsec:hard_label_complexity}Hard labels Increase Data Inefficiency}

In this section we study the splitting effect of training with adding hard labels. Fig.~\ref{fig:4} right shows \(\mathcal{I}(n)\) with respect to different soft ratio \(\rho\).
It also shows a transition from a decreasing curve to constant \(\mathcal{I}(n)\approx 0.5\). This demonstrates that hard labels increase data inefficiency.
This can be explained that discontinuity increases the higher frequency component of the \(z_{\mathrm{s,eff}}(x)\)'s shape, so the shape becomes harder to fit.
The constant of \(\mathcal{I}(n)\) in pure hard label training is lower than that of total random logits (Fig.~\ref{fig:3}). This suggests that learning a discontinuous function is still easier than a random function.

\textbf{Connection with label smoothing.} KD resembles \textit{label smoothing} (LS) technique introduced in \cite{szegedy2016rethinking}.
LS can be regarded as replacing a real teacher with a uniform distribution \(P_{k}=1/K\) and \(y_{\mathrm{LS}} = \epsilon/K + (1-\epsilon)y_{\mathrm{g}}\),
in order to limit student's prediction confidence without the effort of training a teacher.
In our setting of \(K=2\), \(z_{\mathrm{LS,eff}} = a\cdot\mathrm{sgn}(y_{\mathrm{g}})\), where \(a = \log(2/\epsilon - 1)\).
In KD, similarly, student effective logit is approximately a cut-off function (Fig.~\ref{fig:4} left bottom) \(z_{\mathrm{s,eff}} \approx \mathrm{sgn}(z_{\mathrm{t}})\cdot\max\{t, |z_{\mathrm{t}}|\}\),
where the threshold is \(t = z_{\mathrm{s,eff}}(0_+,1)\).
As the soft ratio \(\rho\to 0\), \(t\) exceeds \(\max \{|z_{\mathrm{t}}|\}\) and student's logits tend to a sign function \(z_{\mathrm{s,eff}} \to t\cdot\mathrm{sgn}(z_{\mathrm{g}})\).
Therefore in this limit KD tends to be LS. Results in Fig.~\ref{fig:4} right suggest that LS is not a easy task compared with pure soft KD.
We have to point out the difference between two methods. LS is introduced to  upper limit student's logits \(\max\{|z_{\mathrm{LS}}|\} < a\), while on the contrary, in KD, student's logits are lower limited by a threshold, as we see in Fig.~\ref{fig:4} left bottom.

To conclude, the benefit of KD is that teacher provides student with smoothened output function that is easier to train than with hard labels. This function is easier to fit when teacher is earlier stopped, but at a cost of losing details on class boundary. Adding hard labels damages the smoothness and therefore makes the task more difficult.

\section{\label{sec:imperfect}Case of Imperfect Teacher}

We emphasize that in previous discussion the teacher is assumed to be perfect and teacher's mistake is not considered.
Sec.~\ref{sec:transfer_risk_bound} discusses transfer risk instead of generalization error and in Sec.~\ref{sec:data_inefficiency} we study the difficulty of fitting teacher's output.
In both sections the result shows a favor of pure soft label distillation.
However, according to empirical observations in the original paper of KD \cite{hinton2015distilling}, a small portion of hard labels are required to achieve better generalization error than pure soft ones.
KD on practical datasets and network structure (Fig.~\ref{fig:5} left) also demonstrates this claim.
In this section we analyze benefit of hard label in imperfect teacher distillation.

\textbf{Hard labels correct student's logits at each point.}
The dashed lines in Fig.~\ref{fig:1} left shows \(z_{\mathrm{s,eff}}\) when teacher makes a wrong prediction \(y_{\mathrm{g}}\neq \mathds{1}\{z_{\mathrm{t}} > 0 \}\). From the figure we can see that \(z_{\mathrm{s,eff}}\) moves closer to the correct direction.
To be more precise, we take \(y_{\mathrm{g}} = 1, T = 1.0\) case as an example and solve \(z_{\mathrm{s,eff}}\) according to Eq.~\ref{eq:effective_logits}. No matter how wrong the teacher is,  the existence of hard label restrict the wrongness of student to \(\sigma(z_{\mathrm{s,eff}}) \geq 1 - \rho\). Further, if \(\rho \leq 1/2\), \(z_{\mathrm{s,eff}}\) is always correct. Therefore hard labels can pointwisely alleviate the wrongness of student's logits.

\textbf{Hard labels guide student's weight to correct direction.}
In this section we again assume three oracle weight change of ground truth \(\Delta_{w_\mathrm{g}}\), teacher \( \Delta_{w_{\mathrm{t}}}\) and student \(\Delta_{w_{\mathrm{s,eff}}}\) to express \(z_\mathrm{g}, z_{\mathrm{t}}\) and \(z_{\mathrm{s,eff}}(z_{\mathrm{t}}, \mathrm{sgn}(z_\mathrm{g}))\) (We switch the expression \(\Delta_{w_*}\) to \(\Delta_{w_\mathrm{g}} \) to better distinguish ground truth, teacher and student).
Then student learns a projected weight change of \(\Delta_{\hat{w}} = \mathbf{P}_{\mathbf{\Phi}}\Delta_{w_{\mathrm{s,eff}}}\). For simplicity, we denote \(\mathbf{a}^{\top}\Theta_n^{-1}\mathbf{b} = \langle \mathbf{a}, \mathbf{b} \rangle_{\Theta_n}\) and \(\mathbf{a}^{\top}\mathbf{b} = \langle \mathbf{a}, \mathbf{b} \rangle\).

As we show in Sec.~\ref{sec:transfer_risk_bound}, student with a smaller \(\alpha (\Delta_{\hat{w}}, \Delta_{w_g})\) generalizes better. Therefore we consider the cosine of student and oracle weight change,
\begin{equation}
  \cos \alpha (\Delta_{\hat{w}}, \Delta_{w_\mathrm{g}}) = \frac{\Delta_{w_\mathrm{g}}^{\top} \Delta_{\hat{w}}}{||\Delta_{w_\mathrm{g}}||_2 \sqrt{\Delta_{\hat{w}}^{\top} \Delta_{\hat{w}}}} = \frac{ \langle \Delta_{\mathbf{z}_{\mathrm{g}}}, \Delta_{\mathbf{z}_{\mathrm{s,eff}} } \rangle_{\Theta_n}}{||\Delta_{w_\mathrm{g}}||_2 \sqrt{\langle \Delta_{\mathbf{z}_{\mathrm{s,eff}}}, \Delta_{\mathbf{z}_{\mathrm{s,eff}}} \rangle_{\Theta_n}}} .
\end{equation}
If hard labels are beneficial, then we expect \(\cos \alpha (\Delta_{\hat{w}}, \Delta_{w_\mathrm{g}})\) to increases with respect to \(1-\rho\). 
The dependency of \(\Delta_{w_{\mathrm{s,eff}}}\) on \(\Delta_{w_\mathrm{g}}\) and \(\Delta_{w_{\mathrm{t}}}\) is implicit, so we only consider the effect of hard labels near pure soft  distillation (\(\rho \to 1\)) from an imperfect teacher.
Then the change of \(\cos \alpha (\Delta_{\hat{w}}, \Delta_{w_\mathrm{g}})\) with respect to hard label can be approximated by linear expansion, as summarized by the following theorem.
\begin{theorem}
\textbf{(Effect of Hard Labels)}
We introduce correction logits \(\delta z_{\mathrm{h}}\) to approximate \(z_{\mathrm{s,eff}}\) solved by Eq.~\ref{eq:effective_logits} as a linear combination \(z_{\mathrm{s,eff}} \approx z_{\mathrm{t}} + (1-\rho)\delta z_{\mathrm{h}}\) in the limit of \(\rho \to 1\).
Then the derivative of \( \cos \alpha(\Delta_{\hat{w}}, \Delta_{w_\mathrm{g}})\) with respect to hard ratio \(1-\rho\) is,
\begin{equation}
\label{eq:angle_derivative}
\begin{split}
    \frac{\partial \cos \alpha(\Delta_{\hat{w}}, \Delta_{w_\mathrm{g}})}{ \partial (1- \rho)} \Bigg|_{\rho = 1} = \;\;\;
  &\frac{1}{ || \Delta_{w_\mathrm{g}}||_2 \sqrt{ \langle \Delta_{\mathbf{z}_{\mathrm{t}}}, \Delta_{\mathbf{z}_{\mathrm{t}}} \rangle_{\Theta_n} }} 
 \times \\
 &\left( \langle \Delta_{\mathbf{z}_{\mathrm{g}}}, \delta \mathbf{z}_{\mathrm{h}} \rangle_{\Theta_n} - \frac{\langle \Delta_{\mathbf{z}_{\mathrm{g}}},  \Delta_{\mathbf{z}_{\mathrm{t}}} \rangle_{\Theta_n} }{\langle \Delta_{\mathbf{z}_{\mathrm{t}}},  \Delta_{\mathbf{z}_{\mathrm{t}}} \rangle_{\Theta_n} } \langle \Delta_{\mathbf{z}_{\mathrm{t}}}, \delta \mathbf{z}_{\mathrm{h}} \rangle_{\Theta_n} \right).
\end{split}
\end{equation}
\end{theorem}
The sign of this derivative indicates whether hard label benefits or not. The expression in the parentheses has an intuitive geometric meaning.
It can be written as a projection \(\langle \delta_{\hat{w}_\mathrm{h}}, \Delta_{\hat{w}_{\mathrm{c}}} \rangle\),
where \(\delta_{\hat{w}_\mathrm{h}} = \phi(\mathbf{X})\Theta_n^{-1} \delta z_{\mathrm{h}}\) is the change of student's weight by adding hard labels,
and \(\Delta_{\hat{w}_{\mathrm{c}}} = \Delta_{\hat{w}_{\mathrm{g}}} - \Delta_{\hat{w}_{\mathrm{t}}}\frac{\langle \Delta_{\hat{w}_{\mathrm{t}}}, \Delta_{\hat{w}_{\mathrm{g}}}\rangle}{\langle \Delta_{\hat{w}_{\mathrm{t}}}, \Delta_{\hat{w}_{\mathrm{t}}}\rangle}\) is the orthogonal complement of \(\Delta_{\hat{w}_{\mathrm{t}}} = \phi(\mathbf{X})\Theta_n^{-1}\Delta \mathbf{z}_{\mathrm{t}}\) with respect to \(\Delta_{\hat{w}_{\mathrm{g}}} = \phi(\mathbf{X})\Theta_n^{-1}\Delta \mathbf{z}_{\mathrm{g}}\).
\(\Delta_{\hat{w}_{\mathrm{c}}} \) tells the correct direction where  the student can improve most. Therefore, the sign of this projection \(\langle \delta_{\hat{w}_\mathrm{h}}, \Delta_{\hat{w}_{\mathrm{c}}} \rangle\) shows whether hard labels can lead \(\Delta_{\hat{w}}\) to or against the correct direction.

\begin{figure}[t!]
  \centering
  \includegraphics[width=0.32\textwidth]{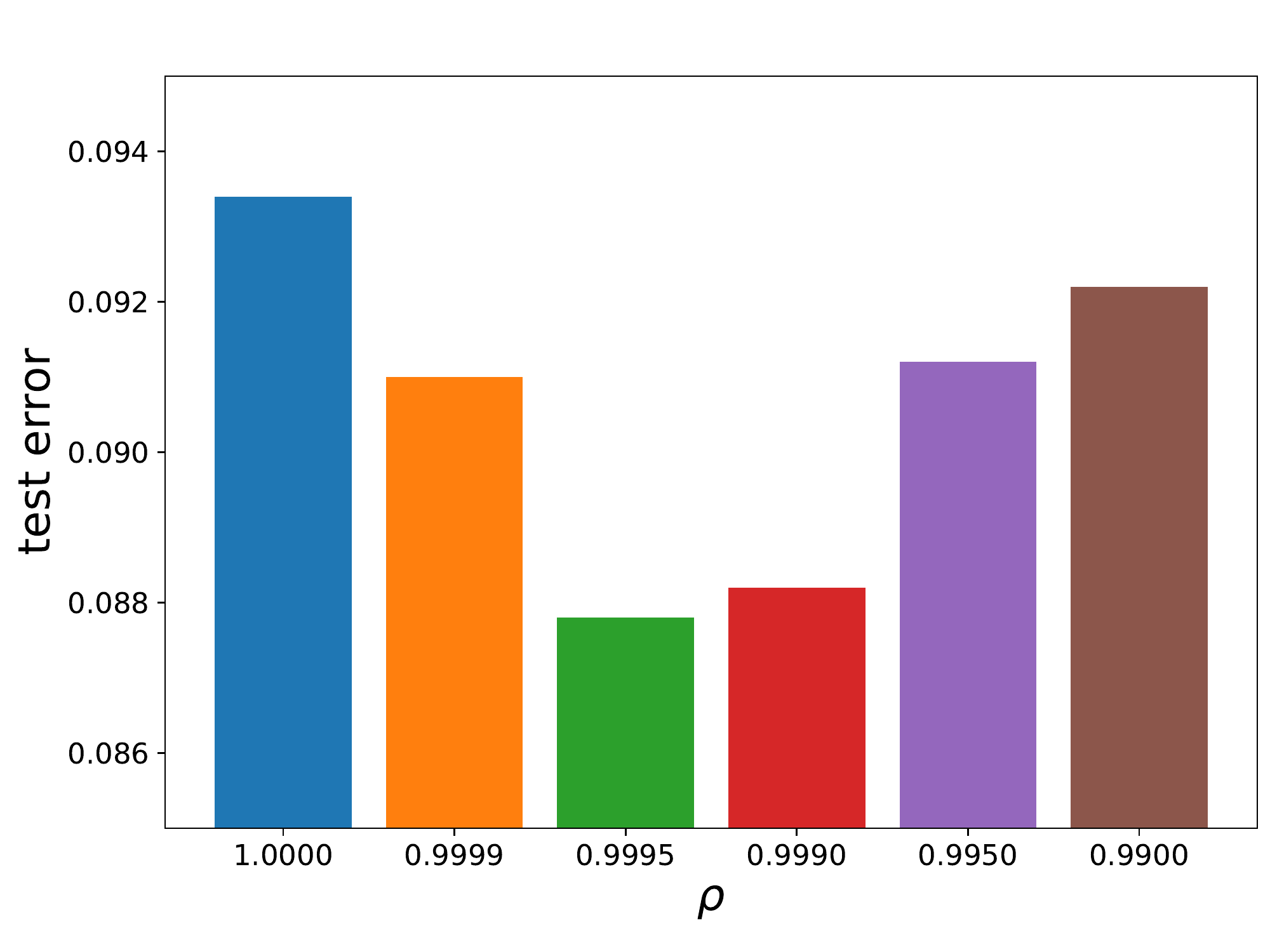}
  \includegraphics[width=0.32\textwidth]{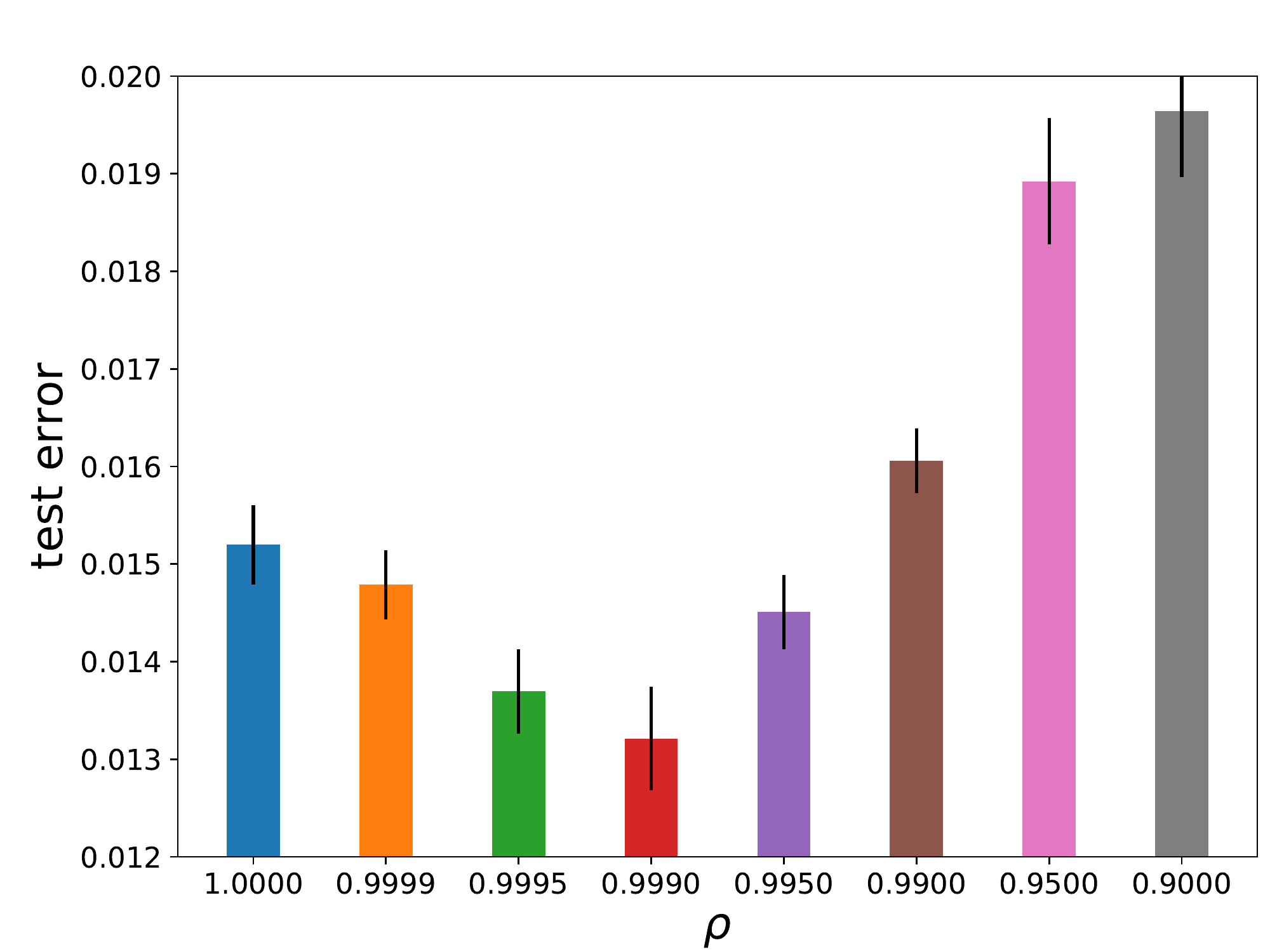}
  \includegraphics[width=0.32\textwidth]{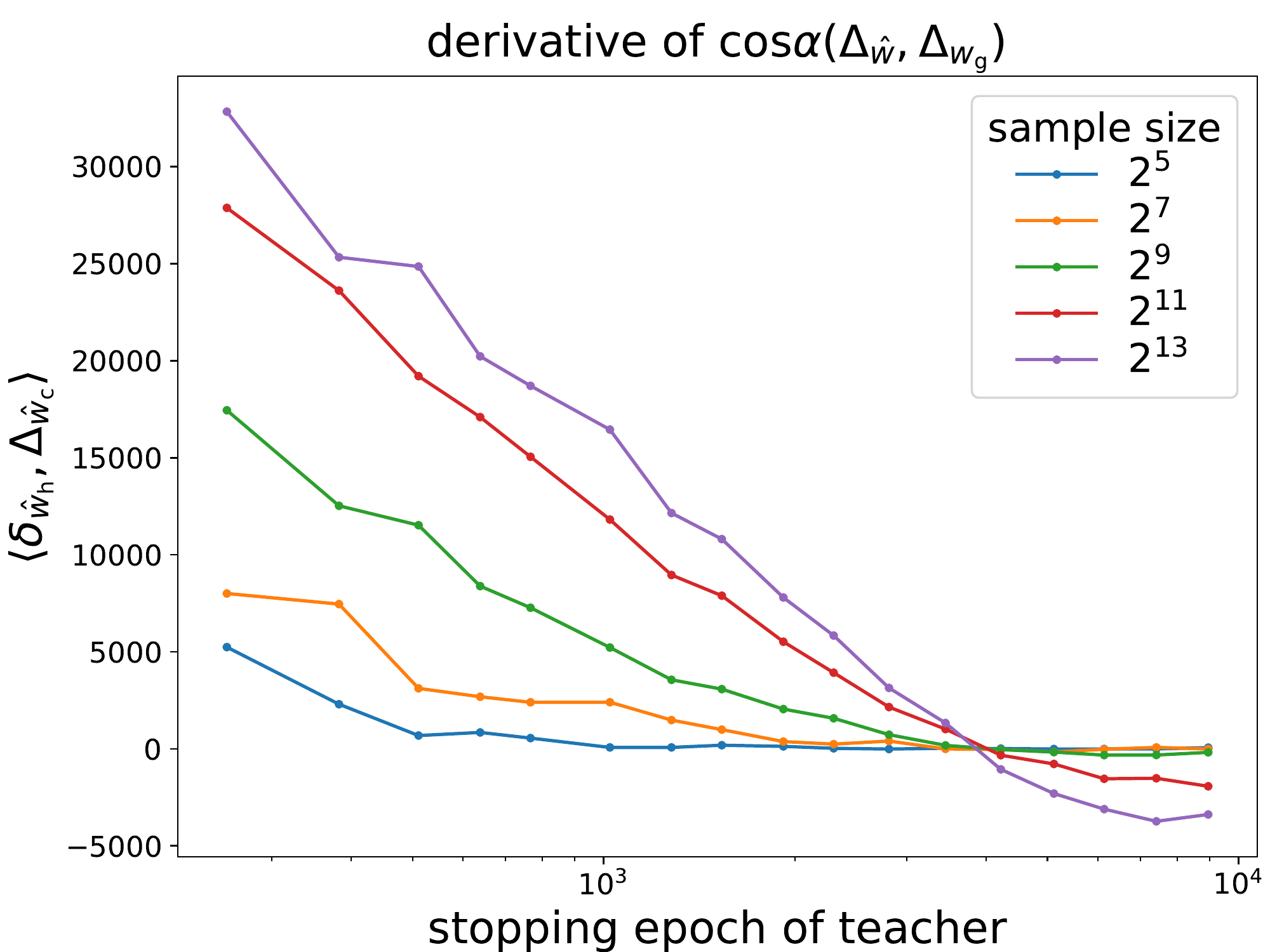}
  \vspace{-0.2cm}
  \caption{\label{fig:5}
  \textbf{Left and Middle:} Imperfect distillation on synthetic dataset and practical (CIFAR10/ResNet) dataset.
  These plots show that pure soft label distillation in imperfect KD is not optimal. \textbf{Right:} \(\langle \delta_{\hat{w}_\mathrm{h}}, \Delta_{\hat{w}_{\mathrm{c}}} \rangle\) is proportional to \(\partial \cos \alpha(\Delta_{\hat{w}}, \Delta_{w_{\mathrm{g}}})/ \partial(1-\rho)\). The sign of it denotes whether adding hard labels can or cannot reduce the angle between the student and oracle. The stopping epoch of teacher is positively related to teacher's generalization ability. The epoch when \(\langle \delta_{\hat{w}_\mathrm{h}}, \Delta_{\hat{w}_{\mathrm{c}}} \rangle\) switches sign, is approximately when teacher outperforms the best of student network. See Sec.~\ref{supp:experiment_detail} for details.}
  \vspace{-0.2cm}
\end{figure}
Fig.~\ref{fig:5} right plots \(\langle \delta_{\hat{w}_\mathrm{h}}, \Delta_{\hat{w}_{\mathrm{c}}} \rangle\) with respect to teacher's stopping epoch \(e\), which is positively related to teacher's generalization ability. Interestingly, when \(e\) is small and teacher is not perfect,  \(\langle \delta_{\hat{w}_\mathrm{h}}, \Delta_{\hat{w}_{\mathrm{c}}} \rangle\) is positive for all sample sizes, but when \(e\) is big enough where the teacher outperforms the best of hard label training student, \(\langle \delta_{\hat{w}_\mathrm{h}}, \Delta_{\hat{w}_{\mathrm{c}}} \rangle\) becomes negative, which means teacher is so accurate that hard labels cannot give useful hint. 
As a short conclusion, adding hard labels is a balance between the inefficiency of data usage, and the correction from teacher's mistakes.

\section{Conclusion}

In this work, we attempt to explain knowledge distillation in the setting of wide network linearization. We give a transfer risk bound based on angle distribution in random feature space. Then we show that, the fast decay of transfer risk for pure soft label distillation, may be caused by the fast decay of student's weight angle with respect to that of oracle model. Then we show that, early stopping of teacher and distillation with a higher soft ratio are both  beneficial in making efficient use of data. Finally, even if hard labels are data inefficient, we demonstrate that they can correct an imperfect teacher's mistakes, and therefore a little portion of hard labels are needed in practical distillation.

For future work, we would like to design new form of distillation loss which does not suffer from discontinuity and at the same time, can still correct teacher's mistakes. We would also like to tighten our transfer risk bound and fit it to a practical nonlinear neural network.

\section*{Broader Impact}

Knowledge distillation is a heavily used technique for model compression. It is of high industrial importance since small models are easy deployed and cost saving. Unfortunately, it remains as a black box. The lack of a theoretical explanation limits a wider application of this technique.
Our work provide a theoretical analysis and rethinking of this technique. Even though this work has no direct societal impact, with the guidance of our analysis, new distillation techniques of higher efficiency and performance may be proposed.
However, development of model compression techniques can have negative impact: it reduces the time and cost of training functional deep learning applications, which make the abuse of deep learning techniques easier. We suggest that deep learning community and governments to put forward stronger restrictions on the opensource of data and models, and more careful supervision on the abuse of computing power.

\section*{Acknowledgement}
This project is supported by The National Defense Basic Scientific Research Project, China (No. JCKY2018204C004), National Natural Science Foundation  of  China  (No.61806009 and 61932001),  PKU-Baidu Funding 2019BD005, Beijing Academy of Artificial  Intelligence  (BAAI). 


\clearpage
\appendix

\newcounter{suppequation}
\newcounter{suppfigure}

\renewcommand{\thesection}{S\arabic{section}}
\renewcommand{\thefigure}{S\arabic{suppfigure}} 
\renewcommand{\theequation}{S\arabic{suppequation}}

\section{\label{supp:distillation_loss_convergence}Convergence of Distillation Loss}

Even though \cite{du2019gradient} only proves the convergence for only L2 loss, we believe this  also holds for our distillation loss, with a little bit of modification. The key idea of \cite{du2019gradient} is that, convergence is guaranteed by the near constancy of NTK matrix \(\hat{\Theta}(\mathbf{X},\mathbf{X})\).
Then we can use the following equation to prove convergence,
\begin{equation}
\stepcounter{suppequation}
    \dot{\mathbf{z}}_{\mathrm{s}} = - \eta \hat{\Theta}(\mathbf{X},\mathbf{X}) (\mathbf{z}_{\mathrm{s}} - \mathbf{z}_{\mathrm{eff}}).
\end{equation}
As proved in the original paper of NTK(\cite{jacot2018neural}), the near constancy of \(\hat{\Theta}(\mathbf{X},\mathbf{X})\) has no requirement on the type of loss, so this is also true for our distillation loss. Then, the difference only lies in the second term \(\mathbf{z}_{\mathrm{s}} - \mathbf{z}_{\mathrm{eff}}\), which in the case of distillation, is substituted  with \(\partial_{\mathbf{z}_{\mathrm{s}}} \mathcal{L}(\mathbf{z}_{\mathrm{s}}, \mathbf{z}_{\mathrm{eff}})\) (same as Eq.~6 in \cite{lee2019wide}). Due to the fact of finite training data and the convexity of \(\mathcal{L}\) (w.r.t \(\mathbf{z}_{\mathrm{s}}\)), the gradient can be lower bounded by another L2 loss, therefore
\begin{equation}
\stepcounter{suppequation}
    |\partial_{\mathbf{z}_{\mathrm{s}}} \mathcal{L}(\mathbf{z}_{\mathrm{s}}, \mathbf{z}_{\mathrm{eff}})| \geq \mu| \mathbf{z}_{\mathrm{s}}- \mathbf{z}_{\mathrm{eff}}|,
\end{equation}
then the convergence of distillation loss can be guaranteed.
A similar proof of convergence is used in Theorem A.~3 of \cite{phuong2019towards} for linear distillation.

\section{\label{supp:Resnet_vary_datanum} Test Error on CIFAR10 with ResNet}

\begin{figure}[htbp]\stepcounter{suppfigure}
  \centering
  \includegraphics[width=0.5\textwidth]{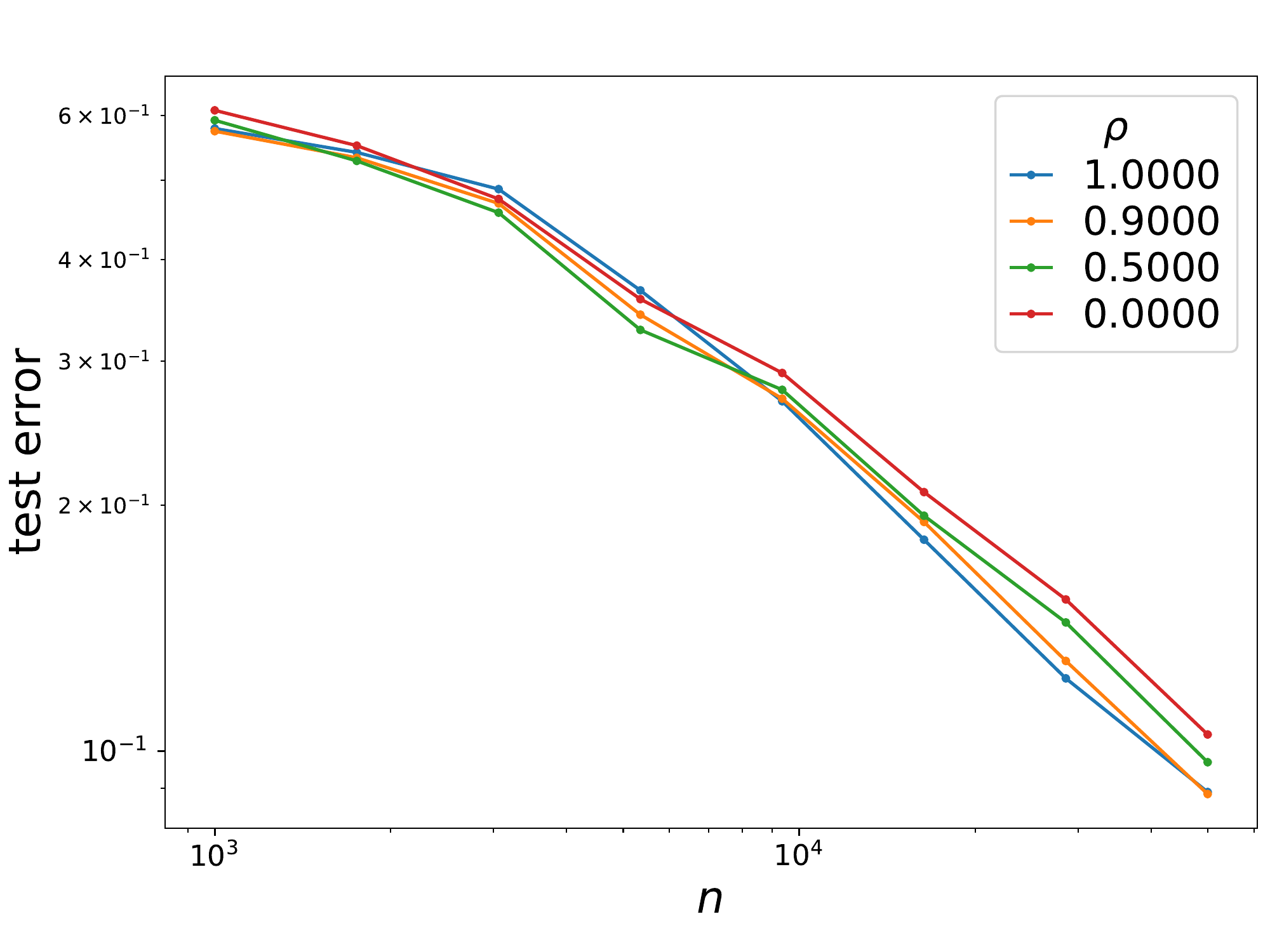}
  \vspace{-0.2cm}
  \caption{\label{fig:s1}
  Test error of knowledge distillation on CIFAR10 dataset with ResNet structure with respect to difference sample size \(n\). 
  }
  \vspace{-0.2cm}
\end{figure}

The settings of Fig.\ref{fig:s1} are same as that of Fig.5 left in the main paper. All data points are averages of 5 times of training, for the purpose of eliminating variations in randomness. The curves shows a near power law relation of test error with respect to sample size \(n\), especially when \(n\) is large. 

\section{\label{supp:proof_of_risk_bound}Proof of Transfer Risk Bound}

One difference of our work from \cite{phuong2019towards} is that the decision boundary contains the initial function as a bias term in NTK linearization \(h(x) = 1 \Leftrightarrow f(x;w_0) + \Delta_w ^{\top}\phi(x) > 0\). We use zero function weight change \(\Delta_{w_{\mathrm{z}}}\) to merge this bias term into part of weight change. Now we restate our risk bound and provide our proof.

\begin{theorem}
  Given input distribution \(P(x)\), training samples \(\mathbf{X} = [x_1, \cdots, x_n]\), oracle weight change \(\Delta_{w_*}\), zero weight change \(\Delta_{w_{\mathrm{z}}}\) and accumulative angle distribution \(p(\beta)\), the transfer risk is bounded by,
  \begin{equation}
  \stepcounter{suppequation}
    R \leq p(\frac{\pi}{2} - \bar{\alpha}_n),
  \end{equation}
  where \(\bar{\alpha}_n = \bar{\alpha}(\Delta_{w_*} - \Delta_{w_{\mathrm{z}}}, \Delta_{\hat{w}} - \Delta_{w_{\mathrm{z}}})\) and \(\Delta_{\hat{w}}\) is student's converged weight change.
\end{theorem}

\begin{proof}

  Denote the oracle weight change and zero weight change as \(\Delta_{w_*}\) and \(\Delta_{w_{\mathrm{z}}}\). Further we denote \(\alpha(a,b) = \cos^{-1}(x^\top y / \sqrt{x^\top x \cdot y^\top y}) \in [0,\pi]\) as a supplement of \(\bar{\alpha}\). The transfer risk can be written as
  \begin{equation}
  \stepcounter{suppequation}
    \begin{split}
      \mathcal{R} &= \underset{x\sim P(x)}{\mathbb{P}}[ (\Delta_{w_*} - \Delta_{w_{\mathrm{z}}})^{\top}\phi(x)\cdot (\Delta_{\hat{w}} - \Delta_{w_{\mathrm{z}}})^{\top}\phi(x)<0 ] \\
      &= \underset{x\sim P(x)}{\mathbb{P}}\left[ \alpha(\phi(x), \Delta_{w_*} - \Delta_{w_{\mathrm{z}}}) < \frac{\pi}{2} , \alpha(\phi(x), \Delta_{\hat{w}} - \Delta_{w_{\mathrm{z}}}) > \frac{\pi}{2} \right] \\
      &+ \underset{x\sim P(x)}{\mathbb{P}}\left[ \alpha(\phi(x), \Delta_{w_*} - \Delta_{w_{\mathrm{z}}}) > \frac{\pi}{2} , \alpha(\phi(x), \Delta_{\hat{w}} - \Delta_{w_{\mathrm{z}}}) < \frac{\pi}{2} \right].
    \end{split}
  \end{equation}
  We further assume \(\alpha_n \leq \pi/2\) so that  \(\alpha_n = \bar{\alpha}_n\). In actual network where \(||\Delta_{w_{\mathrm{z}}}||_2\ll ||\Delta_{w_*}||_2\), this assumption is reasonable, since \(\cos \alpha_n \propto \Delta_{\mathbf{z}}^{\top}\Theta_n^{-1} \Delta_{\mathbf{z}} > 0\).  With the help of triangle inequality, \(\alpha(a,b) \leq \alpha(b,c) + \alpha(c,a)\), we can find a "easy" region where inputs are guaranteed to be correctly classified by student's model. For the case of \(\alpha(\phi(x), \Delta_{w_*} - \Delta_{w_{\mathrm{z}}})< \pi/2\), if we assume the "easy" region is \( \alpha(\phi(x), \Delta_{w_*} - \Delta_{w_{\mathrm{z}}}) < \frac{\pi}{2} - \bar{\alpha}_n\), then
  \begin{equation}\stepcounter{suppequation}
    \begin{split}
      \alpha(\phi(x), \Delta_{\hat{w}} - \Delta_{w_{\mathrm{z}}}) \leq & \alpha(\phi(x),\Delta_{w_*} - \Delta_{w_{\mathrm{z}}}) + \alpha(\Delta_{w_*} - \Delta_{w_{\mathrm{z}}}, \Delta_{\hat{w}} - \Delta_{w_{\mathrm{z}}}) \\
      \leq & \frac{\pi}{2} - \bar{\alpha}_n +  \bar{\alpha}_n = \frac{\pi}{2},
    \end{split}
  \end{equation}
  which means student model also gives a correct prediction. Similarly, \(\alpha(-a,b) \leq \alpha(b,c) + \alpha(c,-a) \Rightarrow \pi - \alpha(a,b) \leq \alpha(b,c) + \pi - \alpha(c,a)\). For the case of \(\alpha(\phi(x), \Delta_{w_*} - \Delta_{w_{\mathrm{z}}})> \pi/2\), if we assume the "easy" region is \( \pi - \alpha(\phi(x), \Delta_{w_*} - \Delta_{w_{\mathrm{z}}}) < \frac{\pi}{2} - \bar{\alpha}_n \), then
  \begin{equation}\stepcounter{suppequation}
    \begin{split}
      \pi - \alpha(\phi(x), \Delta_{\hat{w}} - \Delta_{w_{\mathrm{z}}}) \leq &  \pi - \alpha(\phi(x), \Delta_{w_*} - \Delta_{w_{\mathrm{z}}}) + \alpha(\Delta_{w_*} - \Delta_{w_{\mathrm{z}}}, \Delta_{\hat{w}} - \Delta_{w_{\mathrm{z}}}) \\
      \leq & \frac{\pi}{2} - \bar{\alpha}_n +  \bar{\alpha}_n = \frac{\pi}{2}.
    \end{split}
  \end{equation}

  Then we bound can the transfer risk by the worst case where all \(\phi(x)\) outside this "easy" region is incorrectly classified, so that
  \begin{equation}\stepcounter{suppequation}
    \begin{split}
      \mathcal{R} \leq & \underset{x\sim P(x)}{\mathbb{P}}\left[ \frac{\pi}{2} > \alpha(\phi(x), \Delta_{w_*} - \Delta_{w_{\mathrm{z}}}) > \frac{\pi}{2} - \bar{\alpha}_n \right] \\
      + & \underset{x\sim P(x)}{\mathbb{P}}\left[ \frac{\pi}{2} < \alpha(\phi(x), \Delta_{w_*} - \Delta_{w_{\mathrm{z}}}) < \frac{\pi}{2} + \bar{\alpha}_n \right] \\
      =& \underset{x\sim P(x)}{\mathbb{P}}\left[ \bar{\alpha}(\phi(x), \Delta_{w_*} - \Delta_{w_{\mathrm{z}}}) > \frac{\pi}{2} - \bar{\alpha}_n \right] = p(\frac{\pi}{2} - \bar{\alpha}_n).
    \end{split}
  \end{equation}

\end{proof}

\section{\label{supp:experiment_detail}Experiment Details}

All of our experiments are performed on \texttt{Pytorch 1.5}.

\paragraph{Effective Student Logits Calculation}

\(z_{\mathrm{s,eff}}(z_{\mathrm{t}},y_{\mathrm{g}})\) is attained by solving 
\begin{equation*}
  \lim_{\tau\to\infty} z_{\mathrm{s}}(\tau) = \hat{z}_{\mathrm{s}}, \quad \frac{d \ell}{d \hat{z}_{\mathrm{s}}} = \frac{\rho}{T}(\sigma(\hat{z}_{\mathrm{s}}/T) - \sigma(z_{\mathrm{t}}/T)) + (1-\rho)(\sigma(\hat{z}_{\mathrm{s}}) - y_{\mathrm{g}}) = 0.
\end{equation*}
We use first order gradient descent with a learning rate of \(10.0\) and \(10,000\) iterations to get good results. Second order method like Newton's method is approachable, but it fails when \(|z_{\mathrm{t}}|\) is big and the second order gradient vanishes.

\paragraph{Network Structure}

In our experiment, we  denote input layer as \(d\), and the number of hidden layers as \(L\). We also set the network to have a fixed hidden layer width \(m\). All networks use NTK parameterization and initialization (described in \cite{jacot2018neural}), which has the following form,
\begin{equation*}
  \begin{aligned}
    &h^{1} = \sigma_W W^{0}x/\sqrt{d} + \sigma_b b^{0}, x^{1}=g(h^{1}),\\
    &h^{l+1} = \sigma_W W^{l}x^{l}/\sqrt{m} + \sigma_b b^{l}, \;
    x^{l+1} = g(h^{l+1}),\;
    l=1,2,\cdots, L-1, \\
    &f(x; w) = \sigma_W W^{L}x^{L}/\sqrt{m} + \sigma_b b^{L},\;
  \end{aligned}
\end{equation*}
where \(g(\cdot)\) is the activation function, \(W^{l}_{ij},b^{l}_i\sim \mathcal{N}(0,1)\) are the parameters of each layer, and \((\sigma_W,\sigma_b)\) are the hyperparameters and we use \(\sigma_W=\sigma_b=1.0\) through out our experiments. Further we denote \(w = \cup_{l=0}^{L}\{W^{l}, b^{l}\}\) as the set of parameters.

The linearized network of random features is calculated according to  \(f_{\mathrm{lin}}(x) = f(x;w_0) + \Delta_{w}^{\top}\partial_{w}f(x;w_0)\). 
Practically we calculate the derivative with the help of a \texttt{torch.autograd.functional.jvp} in \texttt{Pytorch 1.5} . All networks are optimized by standard Adam algorithm (\(\beta_1 = 0.9, \beta_2 = 0.999\)) with different learning rate \(\eta\) and batch size \(|\mathcal{D}|\).

\paragraph{Neural Tangent Kernel(NTK) Calculation}

NTK of a ReLU network is calculated according to Appendix C and E in \cite{lee2019wide}.

\paragraph{Input Distribution Design}
In our experiments, the data distribution is fixed to be \(\mathcal{N}(0, 5^2)\).

\paragraph{\label{supp:gaussian_mixture}Gaussian Mixture}

The Gaussian mixture function in Fig.3 has the form of
\begin{equation*}
  z_{\mathrm{gaussian}}(x) = \sum_{j=1}^{q} A_j \exp(-(x-x_j)^2/\sigma_j)
\end{equation*}
\(A_j, x_j, \sigma_j\) are all sampled with randomness, to insure the diversity of Gaussian mixture. \(A_j\) are sampled around a fixed amplitude \(A\), but with equal chance of positive and negative sign. \(x_j\) is sampled according to a gaussian distribution \(\mathcal{N}(0, \sigma_p)\) to make sure most of \(x_j\) are in the distribution of \(x\). \(\sigma_j\) are also sampled around a fixed amplitude \(\sigma\), but \(\sigma\) is changed according to mode number \(q\),  \(\sigma = 15/q^{2}\) so that all points can show their mode in \(z_{\mathrm{gaussian}}(x)\)'s shape.

In the first difficulty control task (Fig.3, left), we control difficulty by control the number of modes \(q\). In the second difficulty control task (Fig.3, right), \(q\) is fixed, and we multiply \(z_{\mathrm{gaussian}}(x)\) with a random variable of \(s \in \{\pm 1\}\), \(z_*(x) = s\times z_{\mathrm{gaussian}}(x) \). \(s\) has probabilities of \(\{1-p_{\mathrm{flip}}, p_{\mathrm{flip}}\}\) to be \(\{1, -1\}\).

In the following Fig.~\ref{fig:s2}, we give a plot of Gaussian mixture function of different mode numbers.

\begin{figure}[htbp]\stepcounter{suppfigure}
  \centering
  \includegraphics[width=0.32\textwidth]{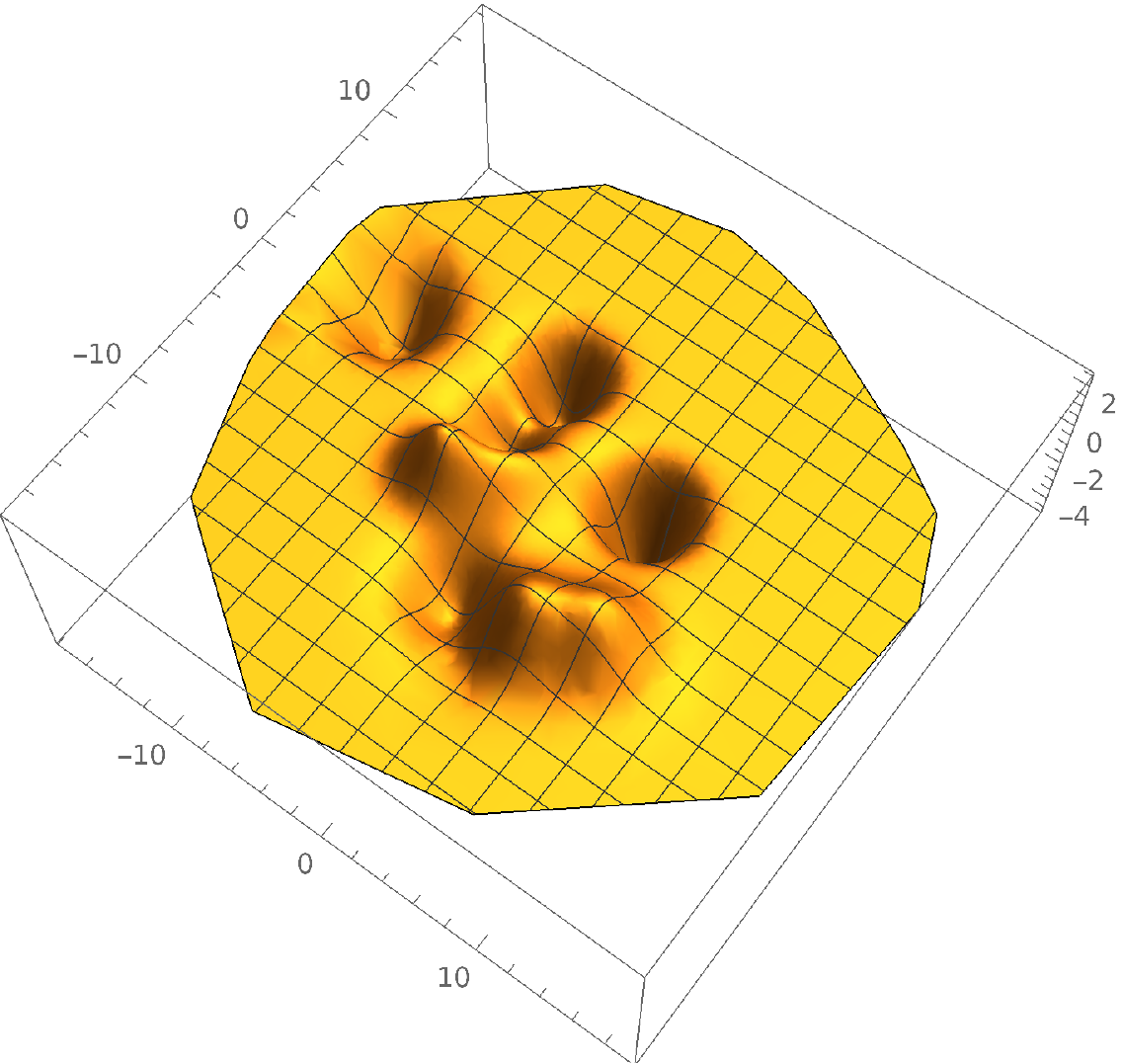}
  \includegraphics[width=0.32\textwidth]{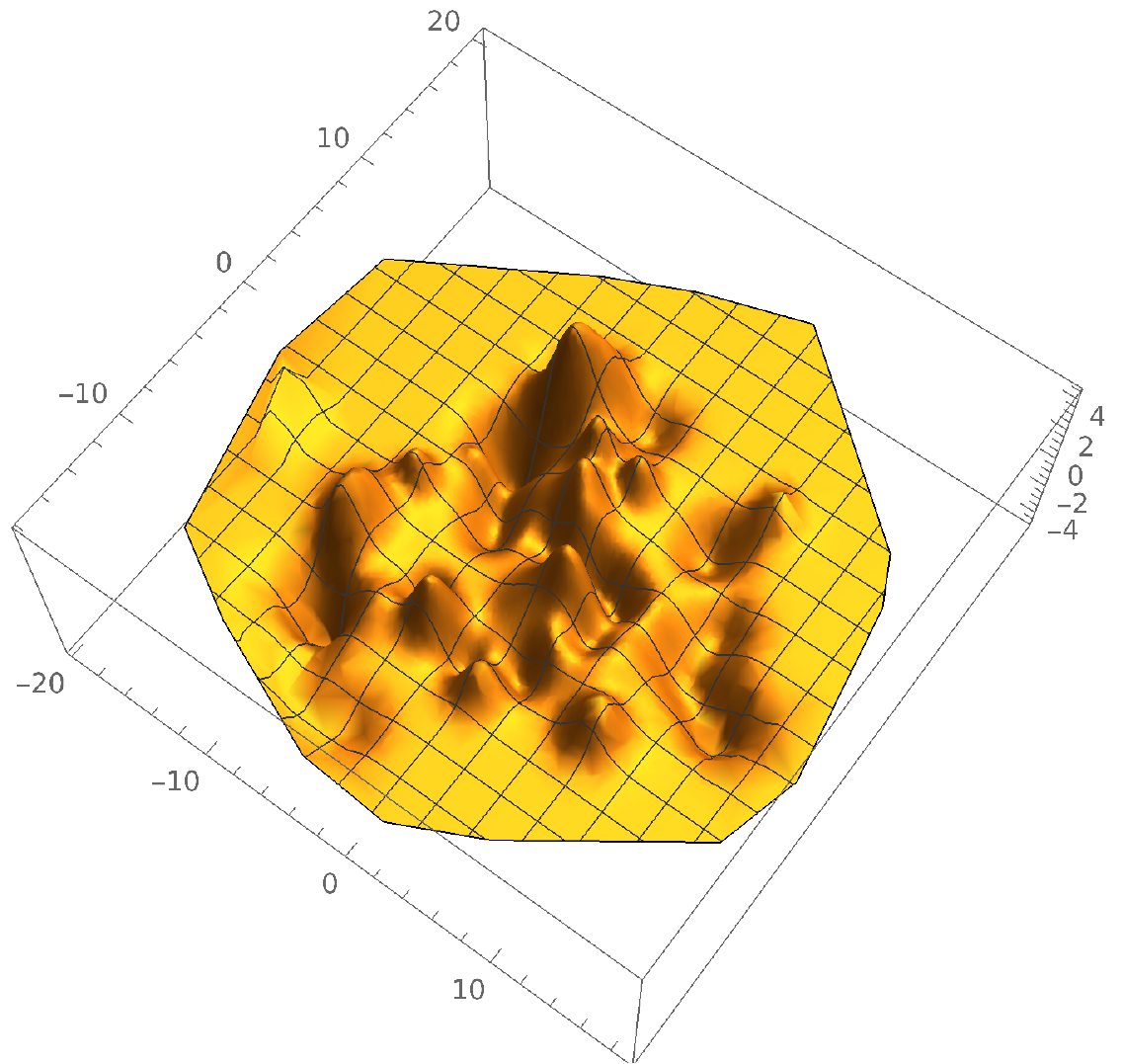}
  \includegraphics[width=0.32\textwidth]{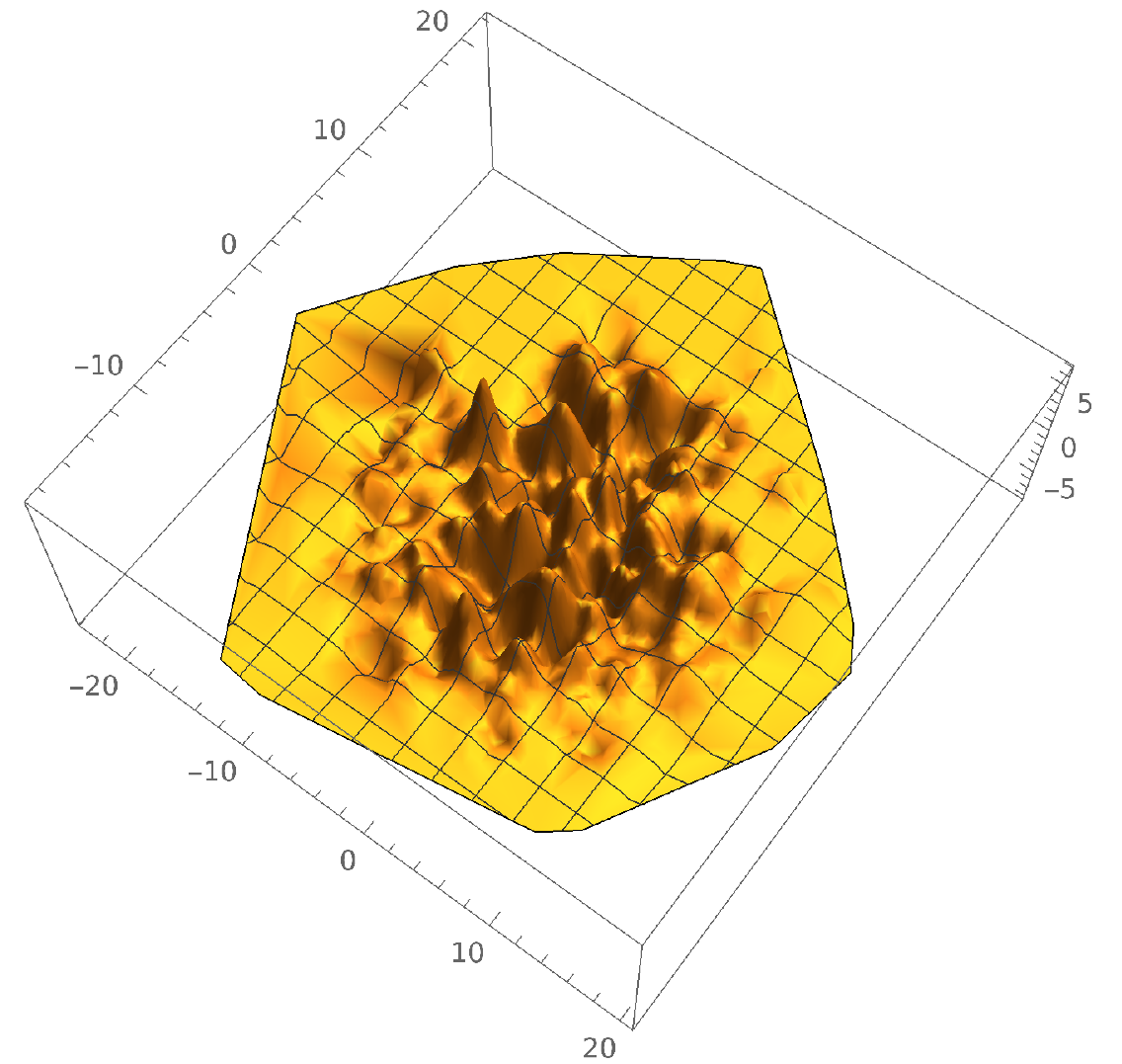}
  \vspace{-0.2cm}
  \caption{\label{fig:s2}
  Examples of Gaussian mixture function. From left to right, each has a mode number of 10, 50, 250.
  }
  \vspace{-0.2cm}
\end{figure}

\paragraph{Calculation of Data Inefficiency}

In the definition of data inefficiency \(\mathcal{I}(n) = n \left[\ln  \mathbb{E}||\Delta_{\hat{w},n+1}||_2 - \ln  \mathbb{E} ||\Delta_{\hat{w},n}||_2\right]\), 
we need to calculate \(||\Delta_{\hat{w},n}||_2 = \sqrt{\Delta_{\mathbf{z}}^{\top}\Theta^{-1}\Delta_{\mathbf{z}}}\). \(\Theta^{-1}\Delta_{\mathbf{z}}\) is calculated by the linear solver \texttt{torch.solve} in \texttt{Pytorch}.

\subsection{Unstated Details of Each Figure}

\paragraph{Fig.2}

The perfect teacher in Fig.2 of main text is the output of a network trained with hard labels generated by a gaussian mixture function. The teacher network has a setting of \(d=2, L=5, m=1024\). It is trained with  learning rate \(\eta = 0.0005\) and batch size \(|D|=4096\). We use \textit{online-batch} training for all teachers, which means samples are regenerated after each epoch, in order to avoid the problem of sample correlation between teacher and student. 

The students in Fig.2 of main text are the linearized network  with a structure of \(d=2, L=3, m=2048\). They are trained with learning rate \(\eta = 0.01\) and batch size \(|D|=512\). The distillation temperature is \(T=10.0\).

We also have to mention that, due to cross entropy loss, the convergence of student logits is especially hard when \(|z_{\mathrm{t}}|\gg 1\). To make our student more easily converge, and to make the scale of teacher match the scale of split generated by hard label \(|z_{\mathrm{t}}| \approx \Delta_{z,\mathrm{split}}\), we reduce the scale of teacher's logits \(z_{\mathrm{t},\mathrm{new}} = r\times z_{\mathrm{t}}\) by a factor \(r\), called \textit{reduction factor}. In our perfect distillation experiment(Fig.2), \(r=0.3\).

Here we give a plot of the decision boundary and output logits of the teacher in Fig.~\ref{fig:s3}.

\begin{figure}[htbp]\stepcounter{suppfigure}
  \centering
  \includegraphics[width=0.49\textwidth]{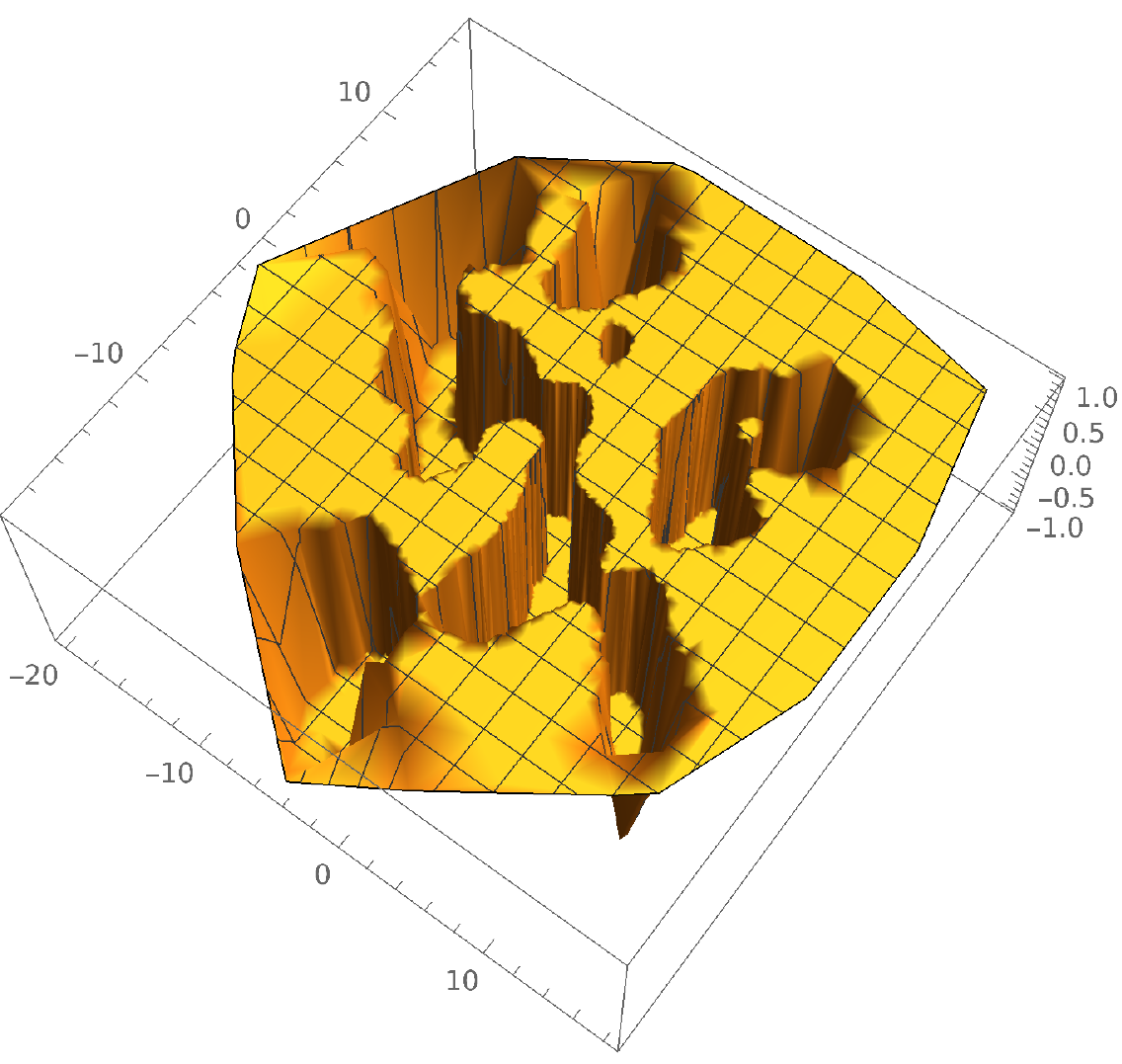}
  \includegraphics[width=0.49\textwidth]{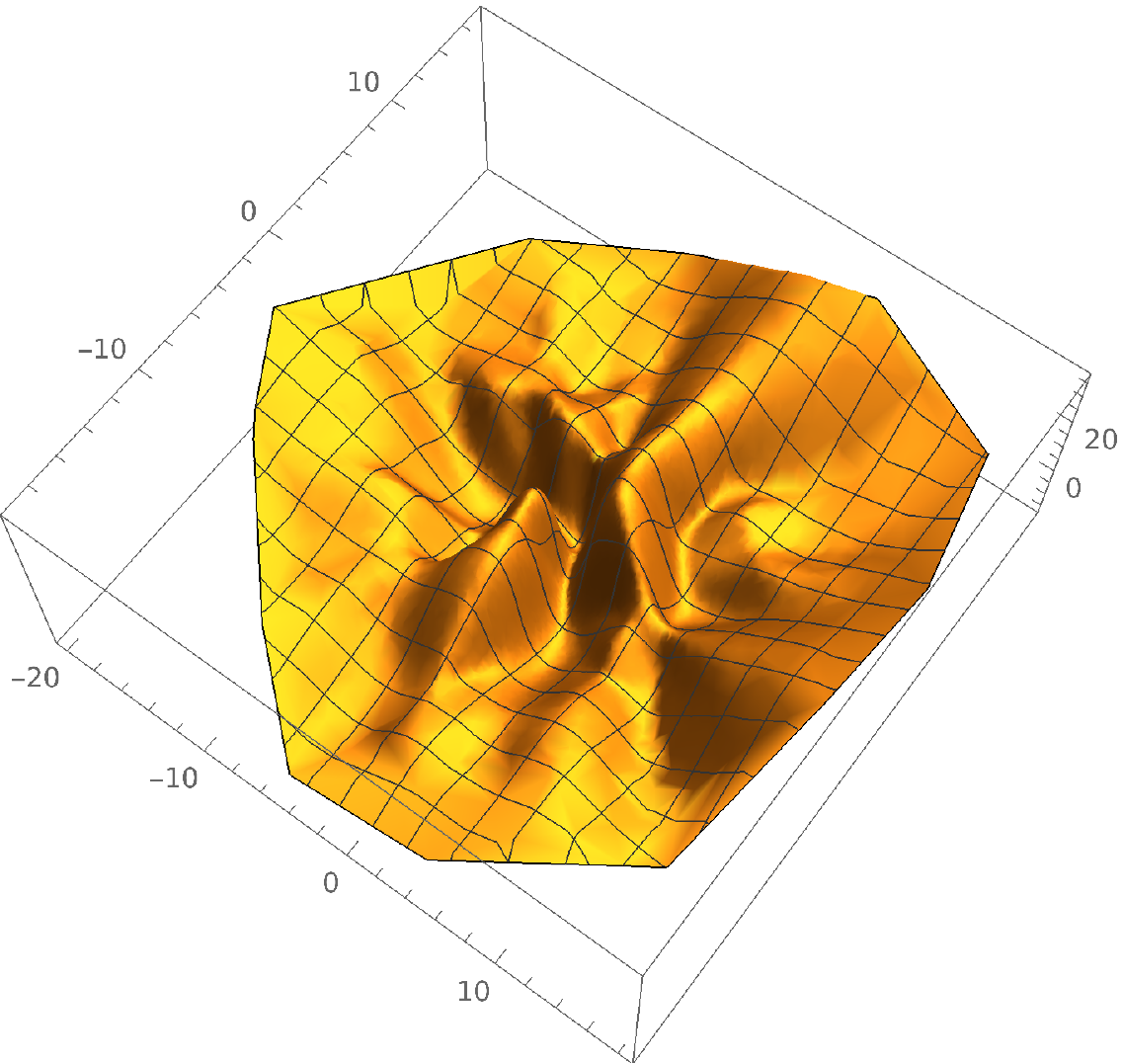}

  \vspace{-0.2cm}
  \caption{\label{fig:s3}
  The decision boundary (left) and output logits(right) of teacher network.
  }
  \vspace{-0.2cm}
\end{figure}

\paragraph{Fig.3}

In Fig.3 of main text, calculation of \(\sqrt{\Delta_{\mathbf{z}}^{\top}\Theta^{-1}\Delta_{\mathbf{z}}}\). \(\Theta^{-1}\Delta_{\mathbf{z}}\) involves the output logits of initialized student network. The initialized network as a structure setting of \(d=2, L=5, m=1024\). 

The data inefficiency is approximately a derivative \(\mathcal{I}(n) = \partial \ln  \mathbb{E} ||\Delta_{\hat{w},n}||_2 / \partial \ln n \). For better illustration, we plot in Fig.~\ref{fig:s4} the intermediate quantity, \(\mathbb{E} ||\Delta_{\hat{w},n}||_2\) with respect to \(n\).

\begin{figure}[htbp]\stepcounter{suppfigure}
  \centering
  \includegraphics[width=0.49\textwidth]{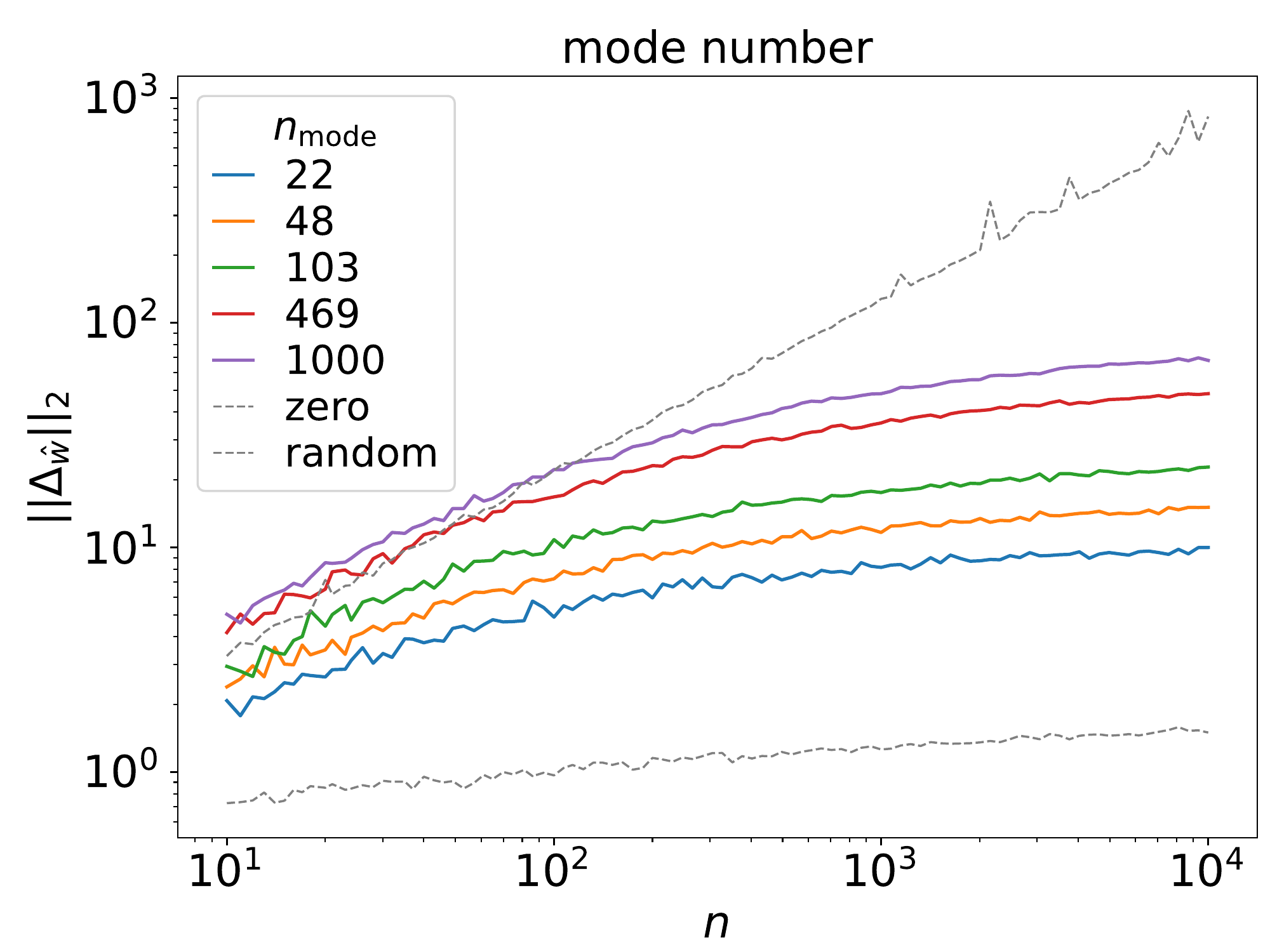}
  \includegraphics[width=0.49\textwidth]{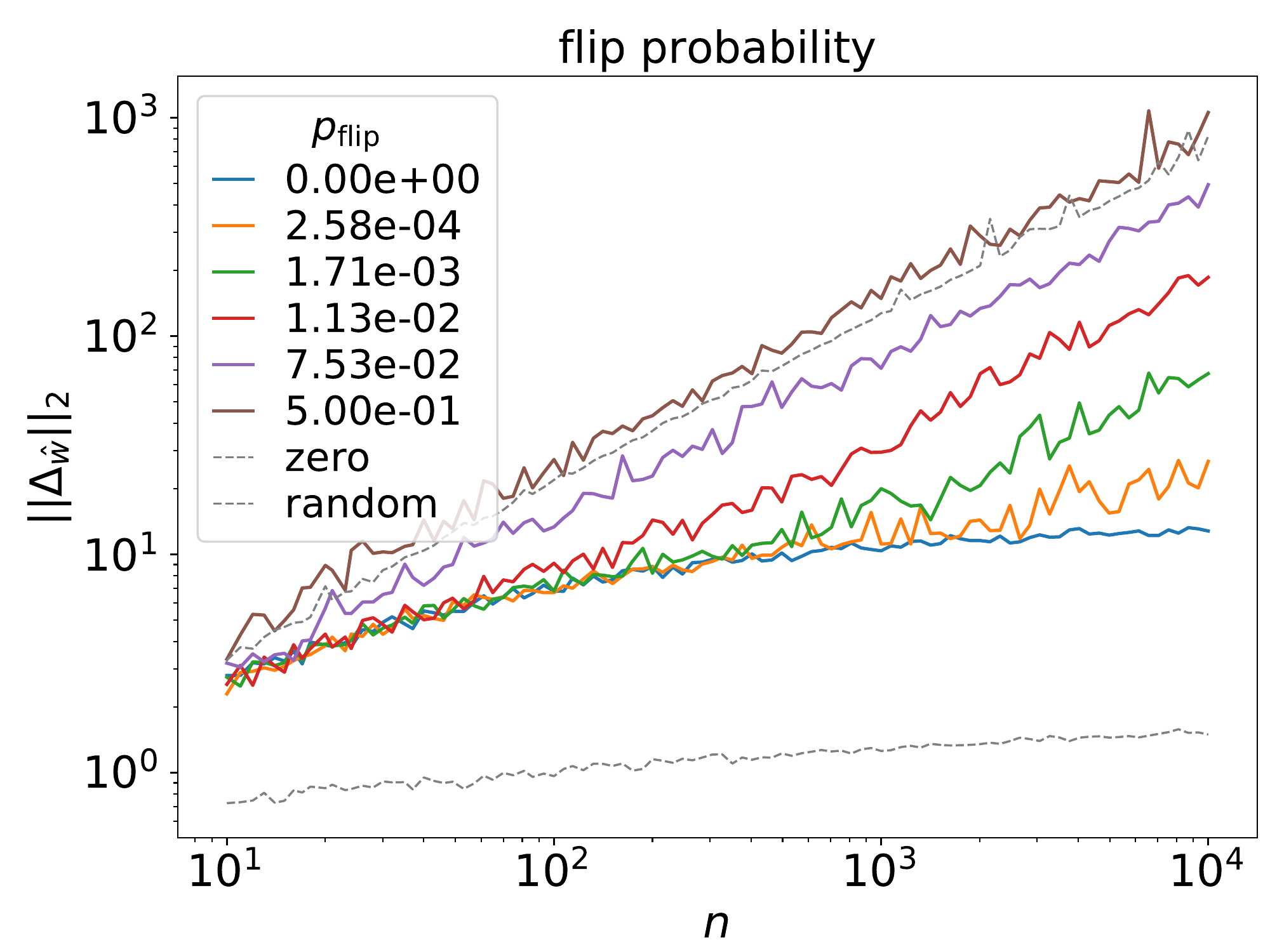}

  \vspace{-0.2cm}
  \caption{\label{fig:s4}
  \(||\Delta_{\hat{w}}||_2\) with respect to sample size \(n\). In the plot \(\Delta_{\mathbf{z}}\) of gaussian mixture functions are normalized by their scale \(\Delta_{\mathbf{z},\mathrm{new}} = \Delta_{\mathbf{z}} / ||\Delta_{\mathbf{z}}||\). The dashed lines are references of random logits(upper dashed lines) and constant zero function(lower dashed lines).
  }
  \vspace{-0.2cm}
\end{figure}

\paragraph{Fig.4}

The teacher of Fig.4 has a network structure setting of \(d=2, L=5, m=1024\). It is trained by the hard label generated by the teacher network of Fig.2 with a learning rate \(\eta = 0.01\) and a total epoch of \(32768\).  In Fig.4 left we use teacher of different stopping epochs to give the plot, while in Fig.4 right we fix the teacher which stops at \(511\)st epoch.

Here we also give plot of \(\mathbb{E} ||\Delta_{\hat{w},n}||_2\) with respect to \(n\) in Fig.~\ref{fig:s5}.

\begin{figure}[htbp]\stepcounter{suppfigure}
  \centering
  \includegraphics[width=0.49\textwidth]{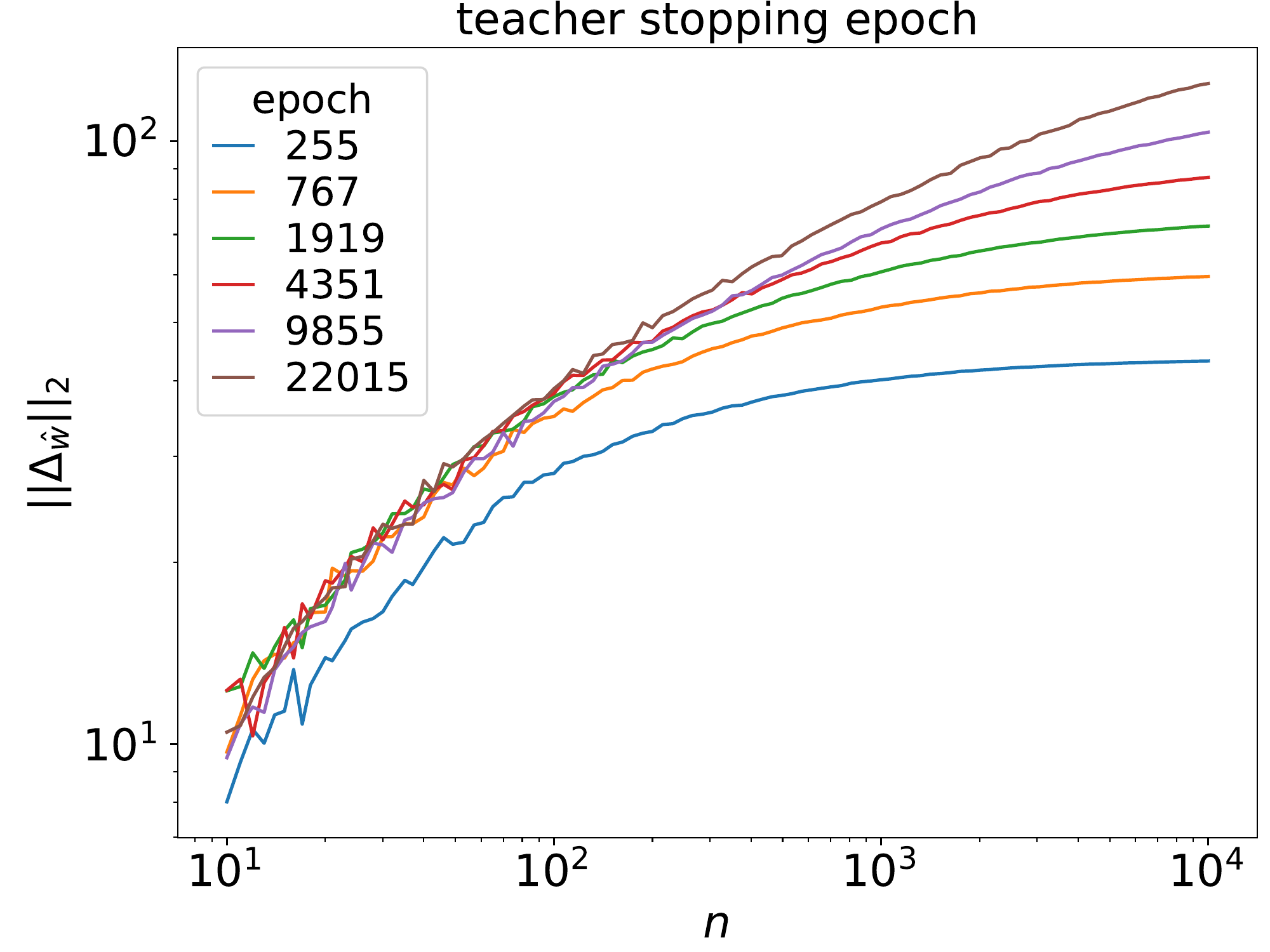}
  \includegraphics[width=0.49\textwidth]{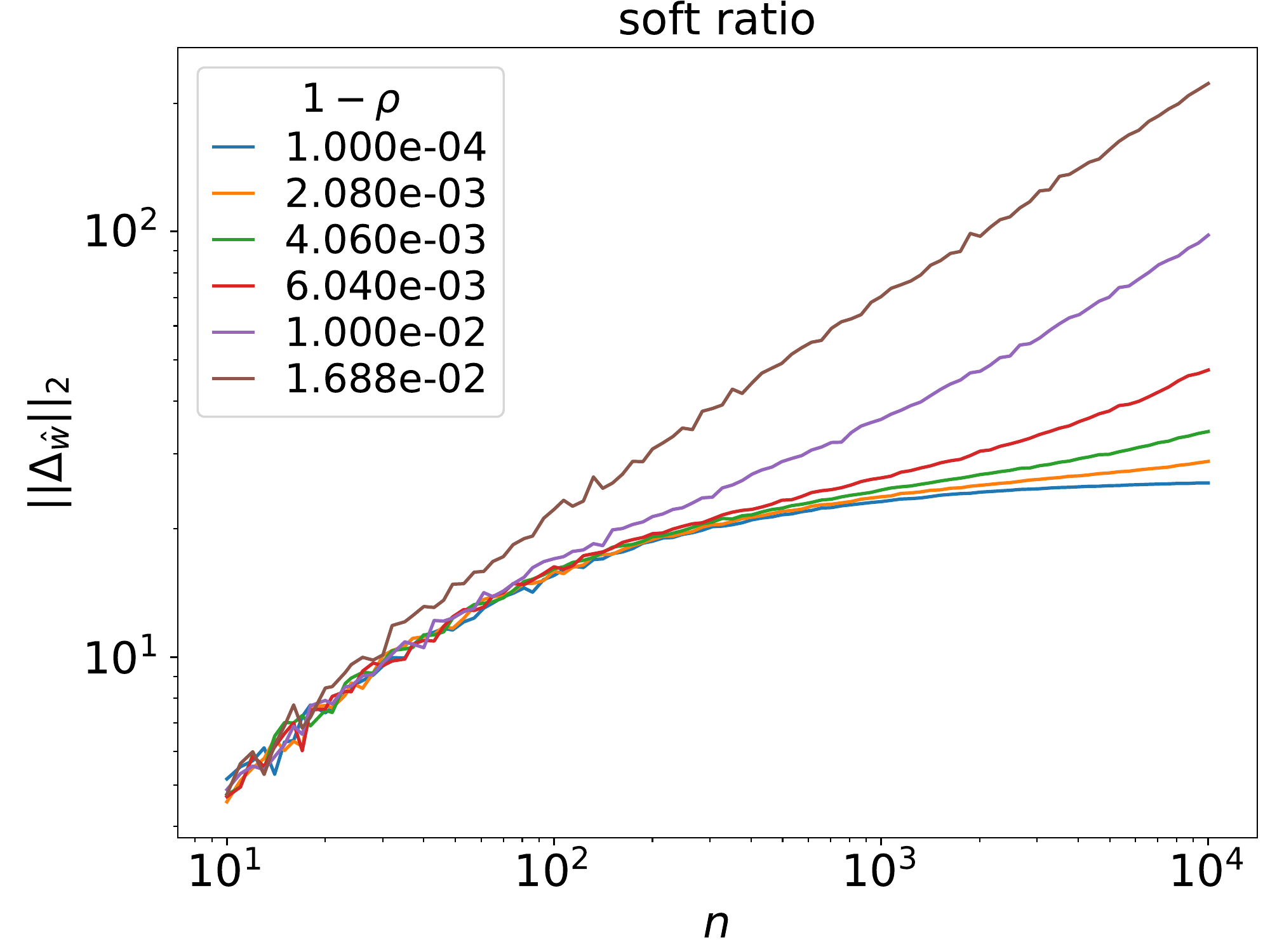}

  \vspace{-0.2cm}
  \caption{\label{fig:s5}
  \(||\Delta_{\hat{w}}||_2\) with respect to sample size \(n\). Similar to Fig.~\ref{fig:s4}, \(\Delta_{\mathbf{z}}\) are normalized.
  }
  \vspace{-0.2cm}
\end{figure}

\paragraph{Fig.5}
\textbf{Left:} The teacher (ResNet50) has a test error of \(6.97\%\). The students are all ResNet18 model, while the students' (ResNet18) baseline test error is \(10.48\%\) if trained from scratch. 
\textbf{Middle:} The teacher is trained by the hard label generated by the teacher network of Fig.2 with a learning rate \(\eta = 0.0001\).
It is early stopped at  \(e=5113\) and at a test error of \(1.06\%\) to make this phenomenon obvious. The sample size for student is \(2^{14}\). 
\textbf{Right:} The teacher is trained in the same way as the Fig.5 middle, but with a total epoch of \(10240\). In the calculation of \(\langle \delta_{\hat{w}_\mathrm{h}}, \Delta_{\hat{w}_{\mathrm{c}}} \rangle\), a reduction factor of \(r=0.3\) is used.

\section{\label{supp:zero_function}Fitting a constant zero function.}

This section is aimed to show that fitting a constant zero function is much easier to train than a normal task. In the following Fig.~\ref{fig:s6}, we give plots on weight change \(\Delta_w\) of fitting both constant zero function and the teacher function in Fig.2 using linearized network of structure \(d=2, L=3, m=2048\). The figure shows that constant zero function is faster to converge, and \(\Delta_{w_{\mathrm{z}}}\) is much smaller than \(\Delta_{w}\) of a normal task.

\begin{figure}[htbp]\stepcounter{suppfigure}
  \centering
  \includegraphics[width=0.5\textwidth]{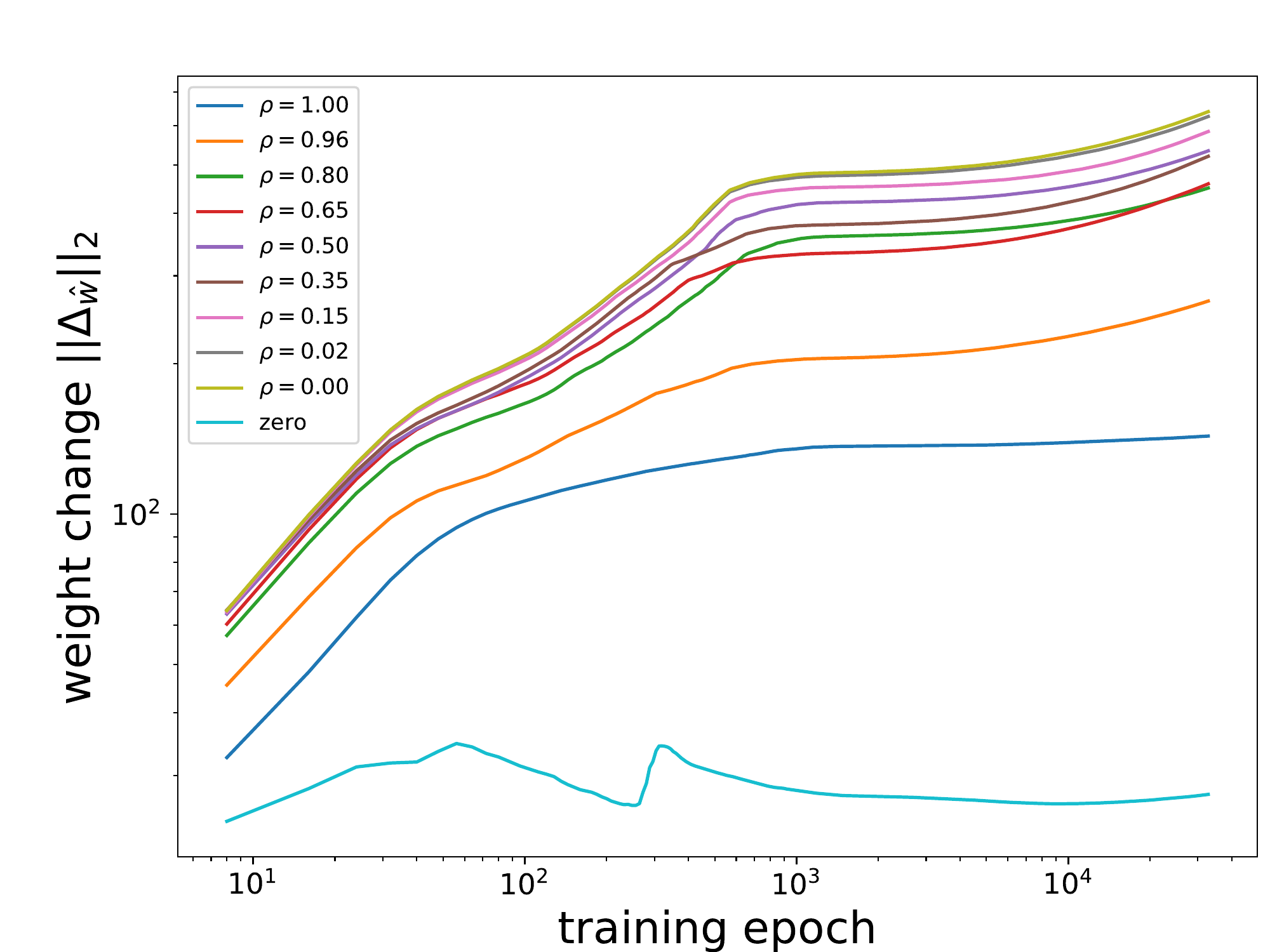}
  \vspace{-0.2cm}
  \caption{\label{fig:s6}
  Weight change plot with respect to training epoch. The bottom curve is that of training zero function, while other curves are distillation tasks, which are plotted as references.
  }
  \vspace{-0.2cm}
\end{figure}

\section{\label{supp:angle_bound} Angle Bound of Random Features of ReLU Network}

This section we aim to show that \cite{phuong2019towards}'s bound \(R_n \leq {\min}_\beta\; p(\beta) + p(\pi/2 - \beta)^n\) is loose in linearized wide neural network. Intuitively, \(p(\beta) \to 1\) when \(\beta \to 0\) and \(p(\beta) \to 0\) when \(\beta \to \pi/2\). If \(p(\beta) < 1\) strictly when \(\beta > 0\), then we can choose a \(\beta \to \pi/2\) that has a relatively small \(p(\beta)\), and with the help of the \(n\) factor in \(p(\pi/2 - \beta)^n\), the total risk bound can approach a small value.

However, as we will show below, the angle of random feature and oracle weight is bounded \(|\cos\bar{\alpha}(\phi(x), \Delta_{w_*} - \Delta_{w_{\mathrm{z}}})| \leq C_1\), so that \(p(\beta)\equiv 1\) strictly for a range of \(\beta \in [0, \beta_t], \beta_t\in (0,\pi/2)\). As we see in Fig.2 middle in the main text, this \(\beta_t\) is probably near \(\pi/2\), which means most random feature \(\phi(x)\) is nearly perpendicular to the oracle. Then the power factor of \(p(\pi/2 - \beta)^n\) will not help reduce the risk bound, so that their risk bound gives an estimate of \(R_n \geq 1\).

To demonstrate this, we use Eq.~7 in the main text to estimate the cosine of angle between random feature and oracle weight, 
\begin{equation*} 
  \cos \bar{\alpha}(\phi(x), \Delta_{w_*} - \Delta_{w_{\mathrm{z}}}) = \frac{( \Delta_{w_*} - \Delta_{w_{\mathrm{z}}})^T\phi(x)}{|| \Delta_{w_*} - \Delta_{w_{\mathrm{z}}}||_2\cdot ||\phi(x)||_2} = \frac{z_{\mathrm{s,eff}}(x)}{|| \Delta_{w_*} - \Delta_{w_{\mathrm{z}}}||_2\cdot \sqrt{\Theta(x,x)}}.
\end{equation*}
In the numerator, \(z_{\mathrm{s,eff}}(x) \sim O(z_{\mathrm{t}}(x))\), especially when \(|z_{\mathrm{t}}(x)| \gg 1\). In knowledge distillation, we assume teacher is also a ReLU network so that the output is also bounded by a linear function, \(|z_{\mathrm{t}}(x))| \leq C_2||x||_2\).
In the dominator, the factor \(|| \Delta_{w_*} - \Delta_{w_{\mathrm{z}}}||_2\) is fixed and we find that for ReLU, the single value neural tangent kernel is \(\Theta(x,x) \sim \Omega(x^\top x)\). Therefore this fraction is bounded \(\cos \bar{\alpha}(\phi(x), \Delta_{w_*} - \Delta_{w_{\mathrm{z}}}) \leq C_1 = C_2/|| \Delta_{w_*} - \Delta_{w_{\mathrm{z}}}||_2\). This \(C_1\) is probably much smaller than 1, as we see in Fig.2 middle of the main text. 

\begin{figure}[htbp]\stepcounter{suppfigure}
  \centering
  \includegraphics[width=0.5\textwidth]{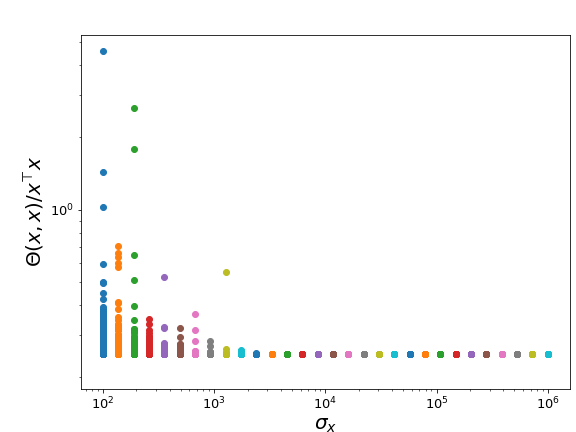}
  \vspace{-0.2cm}
  \caption{\label{fig:s7}
  The NTK \(\Theta(x,x)/(x^{\top}x)\) of a 3-hidden-layer network at different scale of x. 
  We sampled 10000 input smaples according to \(\mathcal{N}(0, \sigma_x^2)\), so that \(\sigma_x\) denotes the scale of input. The figure shows that \(\Theta(x,x) \geq x^{\top}x/4 \) and the inequality tends to equality when the scale is large. This figure demonstrate that \(\Theta(x,x) \sim \Omega(x^\top x)\).
  }
  \vspace{-0.2cm}
\end{figure}

\section{\label{supp:proof_of_theorem2} Proof of Theorem 2}

In this theorem we assume that \(z_{\mathrm{s,eff}} \approx z_{\mathrm{t}} + (1-\rho)\delta z_{\mathrm{h}}\), where \(\delta z_{\mathrm{h}}\) is implicitly determined by hard labels. Then we can approximate 
\begin{equation*}
    \langle \Delta_{\mathbf{z}_{\mathrm{g}}}, \Delta_{\mathbf{z}_{\mathrm{s,eff}} } \rangle_{\Theta_n} \approx
\langle \Delta_{\mathbf{z}_{\mathrm{g}}}, \Delta_{z_{\mathrm{t}}} \rangle_{\Theta_n} + 
(1-\rho)
\langle \Delta_{\mathbf{z}_{\mathrm{g}}}, \Delta_{\delta z_{\mathrm{h}} } \rangle_{\Theta_n},
\end{equation*}
and
\begin{equation*}
\begin{aligned}
    \frac{1}{\sqrt{ \langle \Delta_{z_{\mathrm{s,eff}}}, \Delta_{z_{\mathrm{s,eff}}} \rangle_{\Theta_n} }} &\approx \frac{1}{\sqrt{ \langle \Delta_{z_{\mathrm{t}}}, \Delta_{z_{\mathrm{t}}} \rangle_{\Theta_n} + 2(1-\rho) \langle \Delta_{z_{\mathrm{t}}}, \Delta_{\delta z_{\mathrm{h}}} \rangle_{\Theta_n}}} \\
    &\approx \frac{1}{\sqrt{ \langle \Delta_{z_{\mathrm{t}}}, \Delta_{z_{\mathrm{t}}} \rangle_{\Theta_n}}} - (1-\rho) \frac{\langle \Delta_{z_{\mathrm{t}}}, \Delta_{\delta z_{\mathrm{h}}} \rangle_{\Theta_n}}{\langle \Delta_{z_{\mathrm{t}}}, \Delta_{z_{\mathrm{t}}} \rangle_{\Theta_n}^{3/2}}.
\end{aligned}
\end{equation*}
By substituting the above two terms into Eq.11 of the main paper, and neglect higher order terms, then we can get Eq.12,
\begin{equation}\stepcounter{suppequation}
\begin{split}
    \frac{\partial \cos \alpha(\Delta_{\hat{w}}, \Delta_{w_\mathrm{g}})}{ \partial (1- \rho)} \Bigg|_{\rho = 1} = \;\;\;
  &\frac{1}{ || \Delta_{w_\mathrm{g}}||_2 \sqrt{ \langle \Delta_{\mathbf{z}_{\mathrm{t}}}, \Delta_{\mathbf{z}_{\mathrm{t}}} \rangle_{\Theta_n} }} 
 \times \\
 &\left( \langle \Delta_{\mathbf{z}_{\mathrm{g}}}, \delta \mathbf{z}_{\mathrm{h}} \rangle_{\Theta_n} - \frac{\langle \Delta_{\mathbf{z}_{\mathrm{g}}},  \Delta_{\mathbf{z}_{\mathrm{t}}} \rangle_{\Theta_n} }{\langle \Delta_{\mathbf{z}_{\mathrm{t}}},  \Delta_{\mathbf{z}_{\mathrm{t}}} \rangle_{\Theta_n} } \langle \Delta_{\mathbf{z}_{\mathrm{t}}}, \delta \mathbf{z}_{\mathrm{h}} \rangle_{\Theta_n} \right).
\end{split}
\end{equation}


\begin{thebibliography}{10}

\bibitem{arora2019fine}
Sanjeev Arora, Simon Du, Wei Hu, Zhiyuan Li, and Ruosong Wang.
\newblock Fine-grained analysis of optimization and generalization for
  overparameterized two-layer neural networks.
\newblock In {\em International Conference on Machine Learning}, pages
  322--332, 2019.

\bibitem{barwey2020extracting}
Shivam Barwey, Venkat Raman, and Adam Steinberg.
\newblock Extracting information overlap in simultaneous oh-plif and piv fields
  with neural networks.
\newblock {\em arXiv preprint arXiv:2003.03662}, 2020.

\bibitem{cao2019generalization}
Yuan Cao and Quanquan Gu.
\newblock Generalization bounds of stochastic gradient descent for wide and
  deep neural networks.
\newblock In {\em Advances in Neural Information Processing Systems}, pages
  10835--10845, 2019.

\bibitem{cao2019generalization2}
Yuan Cao and Quanquan Gu.
\newblock Generalization error bounds of gradient descent for learning
  overparameterized deep relu networks.
\newblock {\em arXiv preprint arXiv:1902.01384}, 2019.

\bibitem{cho2019efficacy}
Jang~Hyun Cho and Bharath Hariharan.
\newblock On the efficacy of knowledge distillation.
\newblock In {\em Proceedings of the IEEE International Conference on Computer
  Vision}, pages 4794--4802, 2019.

\bibitem{dong2019distillation}
Bin Dong, Jikai Hou, Yiping Lu, and Zhihua Zhang.
\newblock Distillation $\approx$ early stopping? harvesting dark knowledge
  utilizing anisotropic information retrieval for overparameterized neural
  network.
\newblock {\em arXiv preprint arXiv:1910.01255}, 2019.

\bibitem{du2019gradient}
Simon Du, Jason Lee, Haochuan Li, Liwei Wang, and Xiyu Zhai.
\newblock Gradient descent finds global minima of deep neural networks.
\newblock In {\em International Conference on Machine Learning}, pages
  1675--1685, 2019.

\bibitem{furlanello2018born}
Tommaso Furlanello, Zachary Lipton, Michael Tschannen, Laurent Itti, and Anima
  Anandkumar.
\newblock Born again neural networks.
\newblock In {\em International Conference on Machine Learning}, pages
  1607--1616, 2018.

\bibitem{hinton2015distilling}
Geoffrey Hinton, Oriol Vinyals, and Jeff Dean.
\newblock Distilling the knowledge in a neural network.
\newblock {\em arXiv preprint arXiv:1503.02531}, 2015.

\bibitem{huang2017like}
Zehao Huang and Naiyan Wang.
\newblock Like what you like: Knowledge distill via neuron selectivity
  transfer.
\newblock {\em arXiv preprint arXiv:1707.01219}, 2017.

\bibitem{jacot2018neural}
Arthur Jacot, Franck Gabriel, and Cl{\'e}ment Hongler.
\newblock Neural tangent kernel: Convergence and generalization in neural
  networks.
\newblock In {\em Advances in neural information processing systems}, pages
  8571--8580, 2018.

\bibitem{lassance2020deep}
Carlos Lassance, Myriam Bontonou, Ghouthi~Boukli Hacene, Vincent Gripon, Jian
  Tang, and Antonio Ortega.
\newblock Deep geometric knowledge distillation with graphs.
\newblock In {\em ICASSP 2020-2020 IEEE International Conference on Acoustics,
  Speech and Signal Processing (ICASSP)}, pages 8484--8488. IEEE, 2020.

\bibitem{lee2019wide}
Jaehoon Lee, Lechao Xiao, Samuel Schoenholz, Yasaman Bahri, Roman Novak, Jascha
  Sohl-Dickstein, and Jeffrey Pennington.
\newblock Wide neural networks of any depth evolve as linear models under
  gradient descent.
\newblock In {\em Advances in neural information processing systems}, pages
  8570--8581, 2019.

\bibitem{lee2019graph}
Seunghyun Lee and B~Song.
\newblock Graph-based knowledge distillation by multi-head attention network.
\newblock {\em arXiv preprint arXiv:1907.02226}, 2019.

\bibitem{massart2006risk}
Pascal Massart, {\'E}lodie N{\'e}d{\'e}lec, et~al.
\newblock Risk bounds for statistical learning.
\newblock {\em The Annals of Statistics}, 34(5):2326--2366, 2006.

\bibitem{mobahi2020self}
Hossein Mobahi, Mehrdad Farajtabar, and Peter~L Bartlett.
\newblock Self-distillation amplifies regularization in hilbert space.
\newblock {\em arXiv preprint arXiv:2002.05715}, 2020.

\bibitem{neyshabur2019role}
Behnam Neyshabur, Zhiyuan Li, Srinadh Bhojanapalli, Yann LeCun, and Nathan
  Srebro.
\newblock The role of over-parametrization in generalization of neural
  networks.
\newblock In {\em 7th International Conference on Learning Representations,
  ICLR 2019}, 2019.

\bibitem{phuong2019towards}
Mary Phuong and Christoph Lampert.
\newblock Towards understanding knowledge distillation.
\newblock In {\em International Conference on Machine Learning}, pages
  5142--5151, 2019.

\bibitem{ronen2019convergence}
Basri Ronen, David Jacobs, Yoni Kasten, and Shira Kritchman.
\newblock The convergence rate of neural networks for learned functions of
  different frequencies.
\newblock In {\em Advances in Neural Information Processing Systems}, pages
  4763--4772, 2019.

\bibitem{szegedy2016rethinking}
Christian Szegedy, Vincent Vanhoucke, Sergey Ioffe, Jon Shlens, and Zbigniew
  Wojna.
\newblock Rethinking the inception architecture for computer vision.
\newblock In {\em Proceedings of the IEEE conference on computer vision and
  pattern recognition}, pages 2818--2826, 2016.

\bibitem{tang2020understanding}
Jiaxi Tang, Rakesh Shivanna, Zhe Zhao, Dong Lin, Anima Singh, Ed~H Chi, and
  Sagar Jain.
\newblock Understanding and improving knowledge distillation.
\newblock {\em arXiv preprint arXiv:2002.03532}, 2020.

\bibitem{tsybakov2004optimal}
Alexander~B Tsybakov et~al.
\newblock Optimal aggregation of classifiers in statistical learning.
\newblock {\em The Annals of Statistics}, 32(1):135--166, 2004.

\bibitem{xu2017training}
Zheng Xu, Yen-Chang Hsu, and Jiawei Huang.
\newblock Training shallow and thin networks for acceleration via knowledge
  distillation with conditional adversarial networks.
\newblock {\em arXiv preprint arXiv:1709.00513}, 2017.

\bibitem{xu2019frequency}
Zhi-Qin~John Xu, Yaoyu Zhang, Tao Luo, Yanyang Xiao, and Zheng Ma.
\newblock Frequency principle: Fourier analysis sheds light on deep neural
  networks.
\newblock {\em arXiv preprint arXiv:1901.06523}, 2019.

\bibitem{yoo2019knowledge}
Jaemin Yoo, Minyong Cho, Taebum Kim, and U~Kang.
\newblock Knowledge extraction with no observable data.
\newblock In {\em Advances in Neural Information Processing Systems}, pages
  2701--2710, 2019.

\bibitem{zhang2016understanding}
Chiyuan Zhang, Samy Bengio, Moritz Hardt, Benjamin Recht, and Oriol Vinyals.
\newblock Understanding deep learning requires rethinking generalization.
\newblock {\em arXiv preprint arXiv:1611.03530}, 2016.

\end{thebibliography}
\end{document}